\documentclass[12pt]{article}
\usepackage{times}
\usepackage{graphicx}
\usepackage{color}
\usepackage{multirow}
\usepackage[authoryear]{natbib}
\usepackage{rotating}
\usepackage{bbm}
\usepackage{latexsym}
\usepackage{hyperref}
\usepackage{booktabs} 

\usepackage[utf8]{inputenc} 
\usepackage[T1]{fontenc}    

\usepackage{amssymb,amsmath, amsfonts}
\usepackage{graphicx}
\usepackage{bm}
\usepackage{algorithm}
\usepackage{algorithmic}
\usepackage{lscape}
\usepackage{subfigure}
\usepackage{enumitem}
\usepackage{mathtools}
\usepackage{color}

\newcommand{\argmin}{\mathop{\rm arg~min}\limits}
\DeclareMathOperator\supp{supp}

\newtheorem{theo}{Theorem}[section]
\newtheorem{prop}[theo]{Proposition}
\newtheorem{lemm}{Lemma}[section]
\newtheorem{proof}{Proof}

\newtheorem{defi}{Definition}[section]

\newcommand{\captionfonts}{\normalsize}

\makeatletter  
\long\def\@makecaption#1#2{%
  \vskip\abovecaptionskip
  \sbox\@tempboxa{{\captionfonts #1: #2}}%
  \ifdim \wd\@tempboxa >\hsize
    {\captionfonts #1: #2\par}
  \else
    \hbox to\hsize{\hfil\box\@tempboxa\hfil}%
  \fi
  \vskip\belowcaptionskip}
\makeatother   

\begin{document}
\hspace{13.9cm}1

\ \vspace{20mm}\\

{\LARGE Information Geometrically Generalized Covariate Shift Adaptation}

\ \\
{\bf \large Masanari Kimura$^{\displaystyle 1}$, Hideitsu Hino$^{\dagger, \displaystyle 2, \displaystyle 3}$}\\
{$^{\displaystyle 1}$SOKENDAI, Graduate University for Advanced Studies.\\
Shonan Village, Hayama,
Kanagawa 240-0193 Japan}\\
{$^{\displaystyle 2}$The Institute of Statistical Mathematics.\\
10-3 Midori-cho, Tachikawa, Tokyo 190-8562, Japan}\\
{$^{\displaystyle 3}$Center for 
Advanced Intelligence Project, RIKEN.\\
1-4-1 Nihonbashi, Chuo-ku, Tokyo 103-0027, Japan}\\
$^{\dagger}$corresponding author: mkimura@ism.ac.jp
%


\thispagestyle{empty}
\markboth{}{NC instructions}
\ \vspace{-0mm}\\
%
\begin{center} {\bf Abstract} \end{center}
Many machine learning methods assume that the training and test data follow the same distribution. However, in the real world, this assumption is very often violated. In particular, the phenomenon that the marginal distribution of the data changes is called covariate shift, one of the most important research topics in machine learning. We show that the well-known family of covariate shift adaptation methods is unified in the framework of information geometry. Furthermore, we show that parameter search for geometrically generalized covariate shift adaptation method can be achieved efficiently. Numerical experiments show that our generalization can achieve better performance than the existing methods it encompasses.

\section{Introduction}
When considering supervised learning methods, it is often assumed that the training and test data follow the same distribution~\citep{Bishop1995-mz, Duda2006-rd, Hastie2009-is, Vapnik2013-nm, Mohri2018-rw}.
However, this common assumption is violated in the real world in most cases~\citep{Huang2007-bm,Zadrozny2004-sa,Cortes2008-zd,Quionero-Candela2009-vg,Jiang2008-dd}.

Covariate shift~\citep{Shimodaira2000-vv} is a prevalent setting
for supervised learning in the real world, where the input distribution differs in the training and test phases, but the conditional distribution of the output variable given the input variable remains unchanged.
Covariate shift is a commonly observed phenomenon in real-world machine learning applications, such as emotion recognition~\citep{Hassan2013-nf,Jirayucharoensak2014-df}, 3D pose estimation~\citep{Yamada2012-ov}, brain computer interfaces~\citep{Li2010-pe,Raza2016-py}, spam filtering~\citep{Bickel2009-wq}, and human activity recognition~\citep{Hachiya2012-vo}.
In addition, there has been recent discussion on the relationship between covariate shift and the robustness of deep learning~\citep{Ioffe2015-ns, Arpit2016-ao, Santurkar2018-wy, Nado2020-ts, Huang2020-zx, Awais2020-wv}.

Ordinary empirical risk minimization (ERM)~\citep{vladimirvapnik1998, Vapnik2013-nm} may not generalize well to the test data under covariate shift because of the difference between the training and test distributions.
However, importance weighting for training examples has been shown to be effective in mitigating the effect of covariate shift~\citep{Shimodaira2000-vv, Sugiyama2005-lv, Sugiyama2005-nr, Zadrozny2004-sa}.
The main idea of these strategies is weighting the training loss terms according to their importance, which is the ratio of the training input density to the test input density. The importance weighting is widely adopted even in modern covariate shift studies with deep neural networks (DNN)~\citep{DBLP:conf/nips/FangL0S20,DBLP:journals/sncs/ZhangYLS21}. 

In this paper, we consider the generalization of these methods in the framework of information geometry~\citep{Amari1985-mi, Amari2007-wb}, a tool that allows us to deal with probability distributions on Riemannian manifolds.
This generalization makes it possible to search for good weighting without searching for a large number of parameters.
Our contributions is summarized as follows:
\begin{itemize}
    \item (Section~\ref{subsec:information_geometrically_generalized_iwerm} and ~\ref{subsec:geometric_bias}) We generalize existing methods of covariate shift adaptation in the framework of information geometry. By our information geometrical formulation, geometric biases of conventional methods are elucidated. 
    \item (Section~\ref{subsec:optimization_of_the_generalized_iwerm}) We show that our geometrically generalized covariate shift adaptation method has a much larger solution space than existing methods controlled by only two parameters. Efficient weighting is obtained by searching for parameters using an information criterion or Bayesian optimization.
    \item (Section~\ref{sec:numerical_experiments}) Numerical experiments show that our generalization can achieve better performance than the existing methods it encompasses.
\end{itemize}


\section{Preliminaries}

\subsection{Problem formulation}
First, we formulate the problem of supervised learning.
We denote by $\mathcal{X}\subset\mathbb{R}^d$ the input space.
The output space is denoted by $\mathcal{Y}\subset\mathcal{R}$ (regression) or $\mathcal{Y}\subset\{1,\dots,K\}$ ($K$-class classification).
We assume that training examples $\{(\bm{x}^{tr}_i, y^{tr}_i)\}^{n_{tr}}_{i=1}$ are independently and identically distributed (i.i.d.) according to some fixed but unknown distribution $p_{tr}(\bm{x}, y)$, which can be decomposed into the marginal distribution and the conditional probability distribution, i.e., $p_{tr}(\bm{x},y)=p_{tr}(\bm{x})p_{tr}(y|\bm{x})$.
We also denote the test examples by $\{(\bm{x}^{te}_i,y^{te}_i)\}^{n_{te}}_{i=1}$ drawn from a test distribution $p_{te}(\bm{x}, y) = p_{te}(\bm{x})p_{te}(y|\bm{x})$.

Let $\mathcal{H}$ be a hypothesis class.
The goal of supervised learning is to obtain a hypothesis $h:\mathcal{X}\to\mathbb{R}\ (h\in\mathcal{H})$ with the training examples that minimizes the expected loss over the test distribution:
\begin{equation}
    \mathcal{R}(h) \coloneqq \mathbb{E}_{(\bm{x}^{te},y^{te})\sim p_{te}(\bm{x},y)}\Big[\ell(h(\bm{x}^{te}), y^{te})\Big], \label{eq:expected_loss}
\end{equation}
where $\ell: \mathbb{R}\times\mathcal{Y}\to\mathbb{R}$ is the loss function that measures the discrepancy between the true output value $y$ and the predicted value $\hat{y}\coloneqq h(\bm{x})$.
In this paper, we assume that $\ell$ is bounded from above, i.e., $\ell(y,y')<\infty\ (\forall y,y'\in\mathcal{Y})$.

\begin{defi}[Covariate shift assumption]
\label{def:covariate_shift_adaptation}
We consider that the two distributions $p_{tr}(\bm{x}, y)$ and $p_{te}(\bm{x},y)$ satisfy the covariate shift assumption if the following three conditions hold: 1) $p_{tr}(\bm{x}) \neq p_{te}(\bm{x})$, 2) $\supp(p_{tr}(\bm{x})) \supset \supp(p_{te}(\bm{x}))$ and 3) $p_{tr}(y|\bm{x}) = p_{te}(y|\bm{x})$.
\end{defi}
Under the covariate shift assumption, the goal of covariate shift adaptation is still to obtain a hypothesis $h$ that minimizes the expected loss~\eqref{eq:expected_loss} by utilizing both labeled training examples $\{(\bm{x}^{tr}_i, y^{tr}_i)\}^{n_{tr}}_{i=1}$ and unlabeled test examples $\{(\bm{x}^{te}_i)\}^{n^{te}}_{i=1}$.

\subsection{Previous works}
Ordinary empirical risk minimization (ERM)~\citep{vladimirvapnik1998,Vapnik2013-nm}, a standard approach in supervised learning, may fail under the covariate shift because it assumes that the training and test data follow the same distribution.
Importance weighting has been shown to be effective in mitigating the effect of covariate shift~\citep{Shimodaira2000-vv,Sugiyama2005-lv,Sugiyama2007-lr,Zadrozny2004-sa}:
\begin{equation}
    \min_{h\in\mathcal{H}}\frac{1}{n_{tr}}\sum^{n_{tr}}_{i=1}w(\bm{x}^{tr}_i)\ell(h(\bm{x}^{tr}_i), y^{tr}_i), \label{eq:weighted_empirical_loss}
\end{equation}
where $w:\mathcal{X}\to\mathbb{R}_{\geq 0}$ is a certain weighting function.
\begin{defi}[IWERM~\citep{Shimodaira2000-vv}]
If we choose the density ratio $p_{te}(\bm{x})/p_{tr}(\bm{x})$ as the weighting function, ERM according to \begin{equation}
    \min_{h\in\mathcal{H}}\frac{1}{n_{tr}}\sum^{n_{tr}}_{i=1}\frac{p_{te}(\bm{x}^{tr}_i)}{p_{tr}(\bm{x}^{tr}_i)}\ell(h(\bm{x}^{tr}_i), y^{tr}_i) \label{eq:iwerm}
\end{equation}
has consistency.
\end{defi}
This is called importance weighted ERM (IWERM).
However, IWERM tends to produce an estimator with high variance.
We can reduce the variance by flattening the importance weights, which is called adaptive IWERM (AIWERM):
\begin{defi}[AIWERM~\citep{Shimodaira2000-vv}]
Let $\lambda\in[0,1]$.
If we choose $(p_{te}(\bm{x})/p_{tr}(\bm{x}))^{\lambda}$ as the weighting function, we can obtain the variance-reduced estimator:
\begin{equation}
    \min_{h\in\mathcal{H}}\frac{1}{n_{tr}}\sum^{n_{tr}}_{i=1}\Big(\frac{p_{te}(\bm{x}^{tr}_i)}{p_{tr}(\bm{x}^{tr}_i)}\Big)^\lambda\ell(h(\bm{x}^{tr}_i), y^{tr}_i). \label{eq:aiwerm}
\end{equation}
\end{defi}
Relative IWERM (RIWERM), a stable version of AIWERM, has also been proposed:
\begin{defi}{(RIWERM~\citep{Yamada2011-ws})}
Let $\lambda\in[0,1]$.
If we choose $p_{te}(\bm{x})/\lambda p_{tr}(\bm{x}) + (1-\lambda) p_{te}(\bm{x})$ as the weighting function, we can directly estimate a flattened version of the importance weight:
\begin{equation}
    \min_{h\in\mathcal{H}}\frac{1}{n_{tr}}\sum^{n_{tr}}_{i=1}\frac{p_{te}(\bm{x}^{tr}_i)}{\lambda p_{tr}(\bm{x}^{tr}_i) + (1-\lambda) p_{te}(\bm{x}^{tr}_i)}\ell(h(\bm{x}^{tr}_i), y^{tr}_i). \label{eq:riwerm}
\end{equation}
\end{defi}
All of the above methods are considered as different weighting methods for each point of the training data.
More generally, the method of covariate shift adaptation can be essentially rephrased as a weighting strategy for training data.

\section{Statistical Model and Exponential Family}
Information geometry~\citep{Amari1985-mi, Amari2007-wb} is a powerful framework that allows us to deal with statistical models on Riemannian manifolds. For theoretical investigation, we need the notion of dual connection and curvature tensor associated with Fisher metric, but these details are deferred to the Appendix~\ref{app:mfd} and we here present minimum required definitions and notations. We note that the assumption on the parametric family is only required for the information geometric analysis in Section~\ref{subsec:geometric_bias}. The algorithmic framework of the proposed method is independent of the parametric model.

Since $p_{tr}(y|\bm{x})=p_{te}(y|\bm{x})=p(y|\bm{x})$ from the assumption of Definition~\ref{def:covariate_shift_adaptation},
what we are interested in is the model manifold $(\mathcal{M}, g(\bm{\theta}))$ to which the marginal distribution $p(\bm{x};\bm{\theta})$ belongs:
\begin{equation}
    \mathcal{M} = \Big\{ p(\bm{x}; \bm{\theta})\ ; \bm{\theta}\in\Theta \Big\}.
\end{equation}
Here, $p_{tr}(\bm{x}; \bm{\theta}), p_{te}(\bm{x}; \bm{\theta})\in\mathcal{M}$. We note that elements in $\mathcal{M}$ is specified by its parameter $\bm{\theta}$ and we identify the parameter vector $\bm{\theta}$ to the density function $p(\bm{x}; \bm{\theta})$ and write $p(\bm{x};\bm{\theta}) \simeq \bm{\theta}$ if necessary.
In this paper, we assume that $\mathcal{M}$ is an exponential family and the probability density function can be written as
\begin{equation}
    p(\bm{x}; \bm{\theta}) = \exp\Big\{ \theta^iT_i(\bm{x}) + k(\bm{x}) - \psi(\bm{\theta})\Big\}, \label{eq:exponential_family}
\end{equation}
where $\bm{x}$ is a random variable, $\bm{\theta}=(\theta^1,\dots,\theta^p)$ is an $p$-dimensional vector parameter to specify a distribution, $\bm{T}(\bm{x}) = (T_1(\bm{x}),\dots,T_p(\bm{x}))$ are sufficient statistics of $\bm{x}$, $k(\bm{x})$ is a function of $\bm{x}$ and $\psi$ corresponds to the normalization factor.
In Eq.~\eqref{eq:exponential_family}, and hereafter the Einstein summation convention will be assumed, so that summation will be automatically taken over indices repeated twice in the term, e.g., $\bm{a}^i\bm{b}_i = \sum_{i} \bm{a}^i\bm{b}_i$.

In the exponential family, the natural parameter $\bm{\theta}$ forms the affine coordinate system, i.e.,
\begin{equation}
    \bm{\theta}(t) = (1-t)\bm{\theta}_1 + t\bm{\theta}_2\ \ (\forall \bm{\theta}_1,\bm{\theta}_2\in\Theta,\ \forall t\in[0,1])
\end{equation}
is a geodesic on $\mathcal{M}$. As a dual coordinate of $\bm{\theta}$, the expectation parameter $\bm{\eta}$ is defined by the Legendre transformation
\begin{align*}
    \bm{\eta} =& \nabla\psi(\bm{\theta}),  \quad
    \bm{\theta} = \nabla\varphi(\bm{\eta}),\\  \mbox{where} \; \; 
    \varphi(\bm{\eta}) =& \max_{\bm{\theta}'}\Big\{\bm{\theta}'\cdot\bm{\eta} - \psi(\bm{\theta}')\Big\}.
\end{align*}

Existing weights for covariate shift adaptation are geometrically characterized, then a generalized weight function is designed based on this geometric formulation.

\section{Geometrical Generalization of Covariate Shift Adaptation}
\subsection{Information Geometrically Generalized IWERM}
\label{subsec:information_geometrically_generalized_iwerm}
In order to derive a generalized covariate shift adaptation method, we prepare the following function.
\begin{defi}[$f$-interpolation~\citep{e23050528}]
\label{def:f_interpolation}
For any $a,b,\in\mathbb{R}$, some $\lambda\in[0,1]$ and some $\alpha\in\mathbb{R}$, we define $f$-interpolation as
\begin{equation}
    m_f^{(\lambda,\alpha)}(a,b) = f^{-1}_\alpha\Big\{(1-\lambda) f_{\alpha}(a) + \lambda f_{\alpha}(b) \Big\},
\end{equation}
where
\begin{equation}
    f_\alpha(a) = \begin{cases}
    a^{\frac{1-\alpha}{2}} & (\alpha\neq 1) \\
    \log a & (\alpha = 1)
    \end{cases}
\end{equation}
is the function that defines the $f$-mean~\citep{hardy1952inequalities}.
\end{defi}
We can easily see that this family includes various known weighted means including the $e$-mixture and $m$-mixture for $\alpha=\pm 1$ in the literature of information geometry~\citep{Amari2016-pi}:
\begin{align*}
m_f^{(\lambda,1)}(a,b) =& \exp\{(1-\lambda)\log a + \lambda \log b\},\\
     m_f^{(\lambda, -1)}(a,b) =& (1-\lambda)a + \lambda b,\\ 
  m_f^{(\lambda, 0)}(a,b) =& \Big((1-\lambda)\sqrt{a} + \lambda\sqrt{b}\Big)^2, \\
 m_f^{(\lambda, 3)}(a,b) =& \frac{1}{(1-\lambda)\frac{1}{a} + \lambda\frac{1}{b}}. 
\end{align*}
Also, for any $\bm{u},\bm{v}\in\mathbb{R}^d\ (d>0)$, we write
\begin{align*}
    \bm{m} = m_f^{(\lambda, \alpha)}(\bm{u}, \bm{v}), 
    \mbox{where}
    \quad
    \bm{m}_i = m_f^{(\lambda, \alpha)}(\bm{u}_i, \bm{v}_i).
\end{align*}

Using this function, we generalize the existing methods of covariate shift adaptation.

\begin{lemm}[$f$-representation of AIWERM]
\label{lem:aiwerm}
The marginal positive measures generated by the weighting of AIWERM can be expressed by using the $f$-interpolation function as
\begin{equation}
    p_A^{(\lambda)}(\bm{x}) = m_f^{(\lambda, 1)}(p_{tr}(\bm{x}), p_{te}(\bm{x})).
\end{equation}
\end{lemm}
\begin{proof}
From~Eq~.\eqref{eq:aiwerm}, we consider its expectation as
\begin{align*}
    \hat{h} =& \min_{h\in\mathcal{H}}\int_{\mathcal{X}\times\mathcal{Y}}\Big(\frac{p_{te}(\bm{x})}{p_{tr}(\bm{x})}\Big)^\lambda \ell(h(\bm{x}), y) p_{tr}(\bm{x},y) d\bm{x}dy \\
    =& \min_{h\in\mathcal{H}}\int_{\mathcal{X}\times\mathcal{Y}} \ell(h(\bm{x}), y) p^{(\lambda)}_A(\bm{x})p_{tr}(y|\bm{x}) d\bm{x}dy. 
\end{align*}
Here, 
\begin{align*}
    p^{(\lambda)}_A(\bm{x}) &= \Big(\frac{p_{te}(\bm{x})}{p_{tr}(\bm{x})}\Big)^\lambda p_{tr}(\bm{x})  \\
    \log p^{(\lambda)}_A(\bm{x}) &= \alpha(\log p_{te}(\bm{x}) - \log p_{tr}(\bm{x})) + \log p_{tr}(\bm{x})\\
    =& (1-\lambda)\log p_{tr}(\bm{x}) + \lambda \log p_{te}(\bm{x})  \\
    p^{(\lambda)}_A(\bm{x}) &= \exp\{(1-\lambda)\log p_{tr}(\bm{x}) + \lambda \log p_{te}(\bm{x})\} \label{eq:exponential_interpolation} \\
    =&  m_f^{(\lambda, 1)}(p_{tr}(\bm{x}), p_{te}(\bm{x})).
\end{align*}
\end{proof}

\begin{lemm}[$f$-representation of RIWERM]
\label{lem:riwerm}
The marginal positive measures generated by the weighting of RIWERM can be expressed by using the $f$-interpolation function as
\begin{equation}
    p^{(\lambda)}_R(\bm{x}) = m^{(\lambda, 3)}_f(p_{tr}(\bm{x}), p_{te}(\bm{x})).
\end{equation}
\end{lemm}
\begin{proof}
From Eq.~\eqref{eq:riwerm},
\begin{align*}
    p^{(\lambda)}_R(\bm{x}) &= \frac{p_{te}(\bm{x})p_{tr}(\bm{x})}{\lambda p_{tr}(\bm{x}) + (1-\lambda) p_{te}(\bm{x})}\\
    =& \frac{1}{\lambda\frac{1}{p_{te}(\bm{x})} + (1-\lambda)\frac{1}{p_{tr}(\bm{x})}} 
    = m^{(\lambda, 3)}_f(p_{tr}(\bm{x}), p_{te}(\bm{x})).
\end{align*}
\end{proof}

From the above discussion, the following generalized method of covariate shift adaptation is derived using the $f$-representation.
\begin{theo}[Geometrically generalized IWERM]
For $\lambda\in[0,1]$ and $\alpha\in\mathbb{R}$, AIWERM and RIWERM is generalized as
\begin{equation}
    \hat{h} = \min_{h\in\mathcal{H}}\int_{\mathcal{X}\times\mathcal{Y}}w^{(\lambda,\alpha)}(\bm{x})\ell(h(\bm{x}), y)p_{tr}(\bm{x},y)d\bm{x}dy, \label{eq:giwerm}
\end{equation}
where
\begin{equation}
    w^{(\lambda,\alpha)}(\bm{x}) = 
    \frac{
    m^{(\lambda,\alpha)}_{f}(p_{tr}(\bm{x}),p_{te}(\bm{x}))
    }{p_{tr}(\bm{x})}.
    \label{eq:Gweight}
\end{equation}
\end{theo}
See the Appendix~\ref{app:prf} for the proof.
From Definition~\ref{def:f_interpolation}, we can confirm that
\begin{align*}
    m_f^{(0,\alpha)}(p_{tr}(\bm{x}), p_{te}(\bm{x})) &=p_{tr}(\bm{x}), \; \;
    \mbox{and}\\
    m_f^{(1,\alpha)}(p_{tr}(\bm{x}), p_{te}(\bm{x}))&=p_{te}(\bm{x}),
\end{align*}
for all $\alpha\in\mathbb{R}$, and this means that we can obtain the set of all curves that connect $p_{tr}(\bm{x})$ and $p_{te}(\bm{x})$. 

We note that~\citet{DBLP:journals/sncs/ZhangYLS21} proposed a method based on basis expansion to estimate a flexible importance weight. It is similar to our proposal in the sense that improves the degree of freedom for designing the importance weight. However, our method considers the parametric form of weight, which enables us to achieve information geometric insight.

In many studies of covariate shift problems using the density ratio weighting including~\citet{Yamada2011-ws}, the direct estimation of the density ratio is often employed~\citep{sugiyama2012density}. Our proposed weight function in~\eqref{eq:Gweight} is also represented as density ratio:
\begin{align*}
    w^{(\lambda,\alpha)}(\bm{x})
    =&
    \frac{
    \left[ 
    (1-\lambda) p_{tr}(\bm{x})^{\frac{1-\alpha}{2}}
    +
    \lambda 
    p_{te}(\bm{x})^{\frac{1-\alpha}{2}}
    \right]^{\frac{2}{1-\alpha}}
    }{
    p_{tr}(\bm{x})}
    \\
    =&
    \left[ 
    1-\lambda +
    \lambda 
    \left( 
    \frac{p_{te}(\bm{x})
    }{
    p_{tr}(\bm{x})
    }
    \right)^{\frac{1-\alpha}{2}}
    \right]^{\frac{2}{1-\alpha}}, \quad (\alpha \neq 1).
\end{align*}
It is then also possible to apply the direct estimation of the density ratio using, for example kernel expansion. In our implementation, we simply used the given $p_{tr}(\bm{x})$ and $p_{te}(\bm{x})$ separately because they are explicitly known by the construction of the training and the test datasets as explained in Section~\ref{subsec:inducing_covariate_shift}. In the practical application of the proposed method in which the generative processes of the covariates of training and test data are unknown, direct density estimation would be a promising approach.

\subsection{Geometric Bias}
\label{subsec:geometric_bias}
AIWERM and RIWERM connects two distributions $p_{tr}$ and $p_{te}$ in different ways. Statistical bias and variance of IWERM, AIWERM, and RIWERM are discussed in the respective papers. In this subsection, we study the geometric bias of these methods to have a deeper understanding of these methods from the geometric viewpoint. 

The proposed generalization of IWERM is independent from a specific parametrization of density functions. In this subsection, for theoretical treatment, the exponential  model manifold which contains $p_{tr}(\bm{x};\bm{\theta})$ and $p_{te}(\bm{x};\bm{\theta})$ are considered, hence 
geodesics can be described by a linear combination of parameters as explained in Appendix~\ref{app:mfd}. With this assumption, specifying $\lambda$ and $\alpha$ is equivalent to selecting a point on the geodesic connecting $p_{tr}$ and $p_{te}$.

\begin{defi}[$\alpha$-divergence~\citep{Amari1985-mi}]
Let $\alpha$ be a real parameter.
The $\alpha$-divergence between two probability vectors $\bm{p}$ and $\bm{q}$ is defined as
\begin{equation}
    D_\alpha[\bm{p}:\bm{q}] = \frac{4}{1-\alpha^2}\Big(1 - \sum_i p_i^{\frac{1-\alpha}{2}}q_i^{\frac{1+\alpha}{2}}\Big).
\end{equation}
\end{defi}
\begin{defi}{($\alpha$-representation~\citep{Amari2009-fz})}
For some positive measure $m_i^{\frac{1-\alpha}{2}}$, the coordinate system $\bm{\theta}=(\theta^i)$ derived from the $\alpha$-divergence is 
    $\theta^i = m_i^{\frac{1-\alpha}{2}} = f_\alpha(m_i)$
and denote by $\theta^i$ the $\alpha$-representation of a positive measure $m_i^{\frac{1-\alpha}{2}}$.
\end{defi}
\begin{defi}[$\alpha$-geodesic~\citep{Amari2016-pi}]
\label{def:alpha_geodesic}
The $\alpha$-geodesic connecting two probability vectors $p(\bm{x})$ and $q(\bm{x})$ is defined as
\begin{align}
    \notag
    r_i(\lambda) =& c(t)f^{-1}_{\alpha}\Big\{(1-\lambda)f_{\alpha}(p(x_i)) + \lambda f_\alpha(q(x_i))\Big\},\\
    c(\lambda) =& \left(\sum^p_{i=1}r_i(\lambda)\right)^{-1}.
\end{align}
\end{defi}
Let $\psi_\alpha(\bm{\theta}) = \frac{1-\alpha}{2}\sum_{i=1} m_i$, the dual coordinate system $\bm{\eta}$ is given by $\bm{\eta} = \nabla\psi_\alpha(\bm{\theta})$ as
\begin{equation}
    \eta_i = (\theta^i)^{\frac{1+\alpha}{1-\alpha}} = f_{-\alpha}(m_i),
\end{equation}
which is the $-\alpha$-representation of $m_i$.

From Definitions~\ref{def:f_interpolation} and \ref{def:alpha_geodesic}, we see that $f$-interpoloation is the unnormalized version of the $\alpha$-geodesic.
We write $\tilde{m}^{(\lambda,\alpha)}_f$ for a suitably normalized $f$-interpolation.
The important properties of $\alpha$-geodesics are
\begin{itemize}
    \item the $\alpha$-geodesic is a geodesic in the $\alpha$-coordinate system derived from $\alpha$-divergence,
    \item the $-\alpha$-geodesic is linear in the $-\alpha$-representation.
\end{itemize}

Let $\gamma_c$ be the geodesic connecting two distributions parameterized by $\bm{\theta}_{tr}$ and $\bm{\theta}_{te}$.
Now, we define two types of geometric biases to characterize the dispersion of $\bm{\theta}_{tr}$ from $\bm{\theta}_{te}$ with respect to the direction along the $\alpha$-geodesic and to the direction orthogonal to the $\alpha$-geodesic.

\begin{defi}[Geodesic bias and curvature bias]
If we write the unit vector along the $\alpha$-geodesic direction as $e_1$ and any unit vector in the orthogonal direction to $e_1$ as $e_2$, the bias relative to the test distribution due to weighting can be decomposed as follows:
\begin{itemize}
    \item geodesic bias: $b_g = (1-\lambda)e_1$,
    \item curvature bias: $b_c = (1-\lambda)tr_g(\mathrm{Ric})e_2$,
\end{itemize}
where $tr_g$ is the trace operation on the metric tensor $g$ and $\mathrm{Ric}$ is the Ricci curvature of the curve connecting the two points generated by the weighting:
\begin{equation}
    \mathrm{Ric} = R_{ikj} d\bm{\theta}^i \otimes d\bm{\theta}^j.
\end{equation}
Here, $R_{ikj}$ is the Riemannian curvature tensor.
\end{defi}
For more detail on the geometric concepts, see textbooks on Riemannian manifolds~\citep{Jost2017-ad}. This definition of geometric biases is consistent with the fact that IWERM, which corresponds to $\lambda=1$, leads to an unbiased estimator of the risk in the test dataset.

\begin{prop}
\label{prop:geometric_bias_aiwerm}
For AIWERM, the geometric bias $b_A(\lambda)$ is computed as
\begin{equation}
    b_A(\lambda) = (1-\lambda)e_1.
\end{equation}
\end{prop}

\begin{figure}
    \centering
    \includegraphics[width=0.98\linewidth,bb=0 370 590 720]{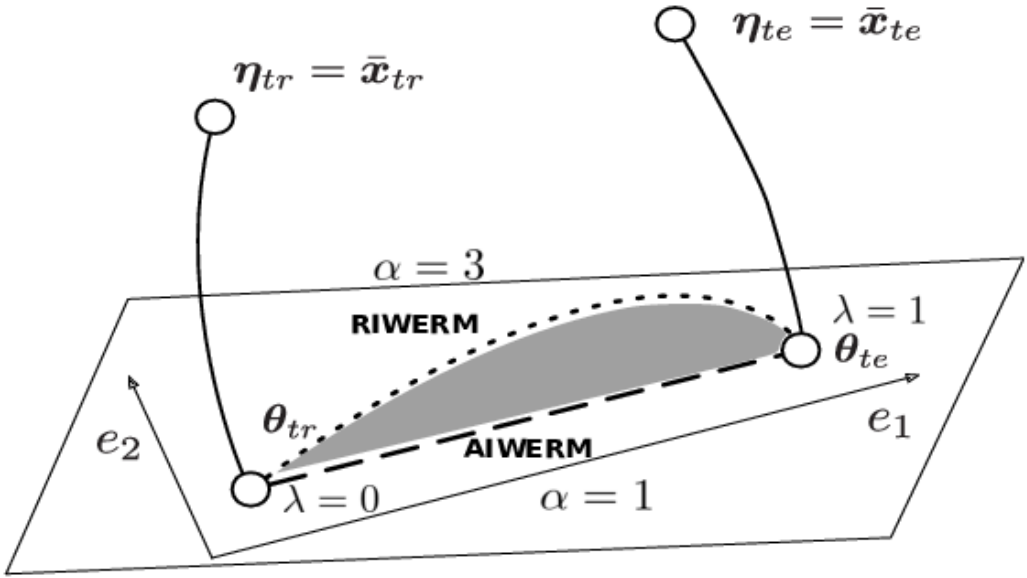}
    \caption{Geometry of covariate shift adaptation methods.
    In the $\bm{\theta}$-coordinate system, the dashed line corresponds to AIWERM and the dotted line corresponds to RIWERM.
    We write unit vector along the $\alpha$-geodesic direction as $e_1$ and any unit vector in the orthogonal direction to $e_1$ as $e_2$.
    Here, $\lambda=0$ and $\lambda=1$ correspond to $\theta_{tr}$ (ERM) and $\theta_{te}$ (IWERM), respectively, and $\alpha=1$ and $\alpha=3$ correspond to the AIWERM and RIWERM curves in the figure.}
    \label{fig:geometry_of_covariate_shift}
\end{figure}
\begin{figure}
    \centering
    \includegraphics[width=0.98\linewidth,bb=0 850 960 1420]{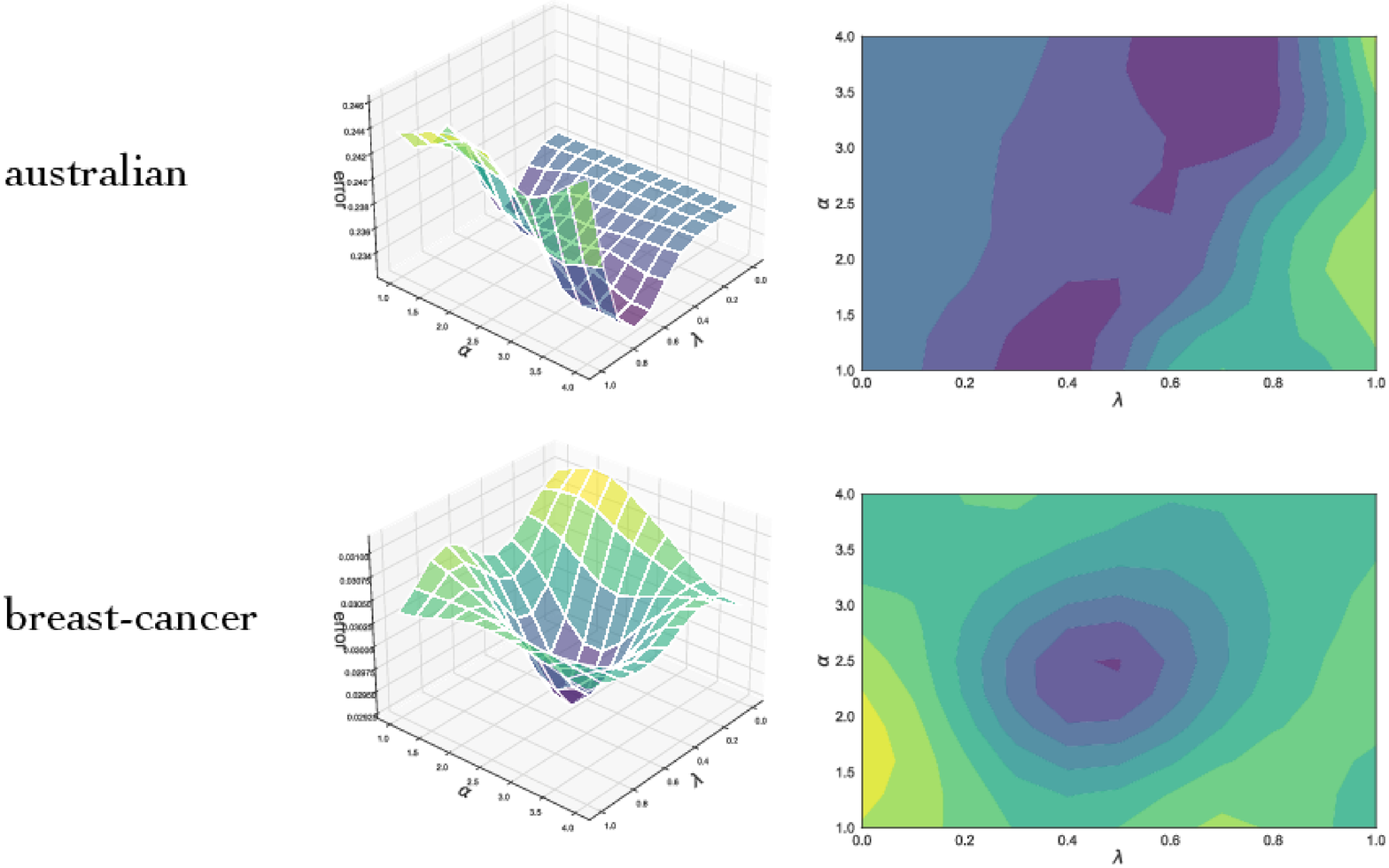}
    \caption{Visualization of grid search for $\alpha$ and $\lambda$ on LIBSVM dataset.}
\end{figure}

\begin{algorithm}[t]
\caption{Bayesian optimization for IGIWERM}         
\label{alg:bopt}                          
\begin{algorithmic}                  
\REQUIRE acquisition function $a(\lambda,\alpha|D)$, target function $L(h;\lambda, \alpha)$, initial points $D_{init}$ compose of a set of parameters $\Xi=\{(\lambda,\alpha)\}$ and corresponding values of the target function
\ENSURE $(\lambda^\ast, \alpha^\ast)$ that minimizes $\min_{h \in \mathcal{H}} L(h;\lambda, \alpha)$
\STATE Initialize $D=D_{init}$
\WHILE{Not converge}
\STATE $\hat{\lambda}, \hat{\alpha} = \argmin_{\lambda, \alpha} a(\lambda,\alpha|D),\; \; \Xi = \Xi \cup \{(\hat{\lambda},\hat{\alpha})\}$
\STATE $\hat{e} = L(h;\hat{\lambda}, \hat{\alpha}),\; \;D = D\cup\{(\hat{\lambda}, \hat{\alpha},\hat{e})\}$
\ENDWHILE
\STATE $(\lambda^\ast, \alpha^\ast) = \argmin_{(\lambda, \alpha) \in \Xi} \left\{\min_{h \in \mathcal{H}} L(h;\lambda, \alpha)\right\}$
\end{algorithmic}
\end{algorithm}

\begin{prop}
\label{prop:geometric_bias_riwerm}
For RIWERM, the geometric bias $b_R(\lambda)$ is computed as
\begin{equation}
    b_R(\lambda) = (1-\lambda)\Big\{e_1 + tr_g\Big(-4\Lambda_{ikj}d\bm{\theta}^i \otimes d\bm{\theta}^j\Big)e_2\Big\}.
\end{equation}
Here, $\Lambda$ is a tensor that depends on the connection.
\end{prop}
These propositions are proved by straightforward calculation as detailed in Appendix~\ref{app:prf}

\begin{figure}[t]
    \centering
    \includegraphics[width=1.00\linewidth, bb=0 645 914 918]{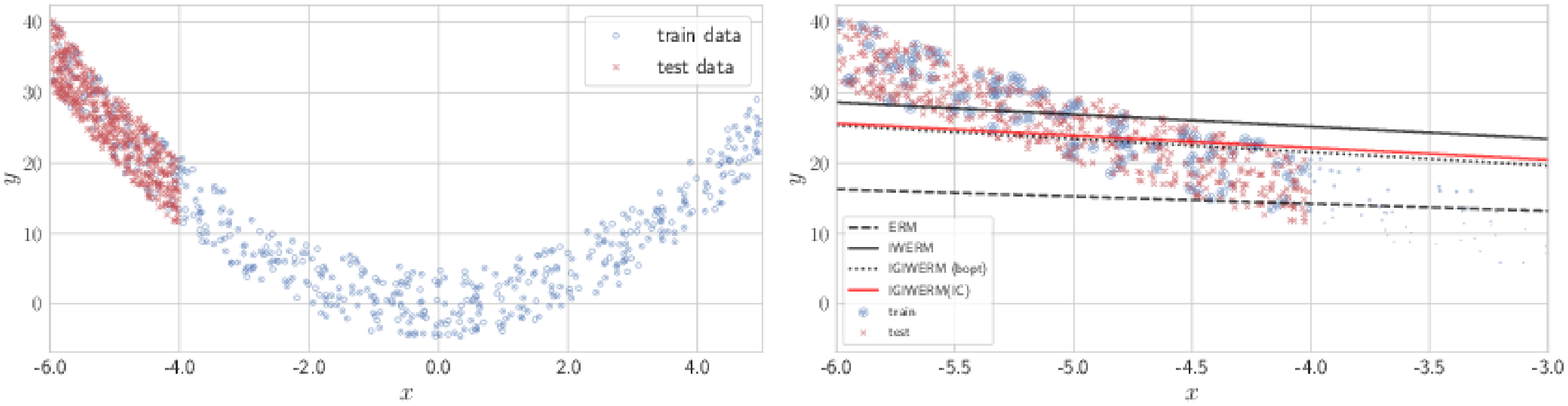}
    \caption{Left: generated data from $y=x^2+\varepsilon$. We see that $p_{tr}(x)$ and $p_{te}(x)$ are different. Right: results of fitting by ERM, IWERM, and IGIWERM.}
    \label{fig:dummy_experiment}
\end{figure}

\begin{figure}[t]
    \centering
    \includegraphics[width=0.99\linewidth,bb=0 0 1150 570]{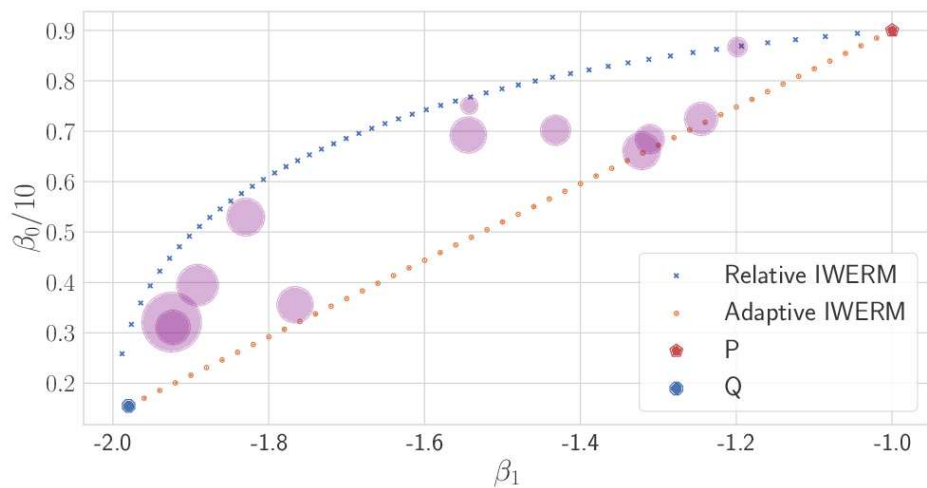}
    \caption{Bayesian optimization for IGIWERM. The coordinates of the purple circles are the parameters explored by Bayesian optimization, and the size of the purple circles indicates the goodness of the parameters (inverse of the MSE).\label{fig:bopt}}
\end{figure}
\begin{table}[t]
    \centering
\caption{Mean squared errors of covariate shift adaptation methods in regression problems over $10$ trials. Here, IGIWERM (bopt) is the Bayesian optimization based, and IGIWERM (IC) is the information criterion based strategy.\label{tab:dummy_mse}}    
    \begin{tabular}{l|l}
    Weighting strategy & MSE    \\ \hline
    ERM                & $160.19 (\pm 4.25)$       \\
    IWERM              & $33.76 (\pm 3.82)$        \\
    AIWERM             & $31.14 (\pm 2.97)$        \\
    RIWERM             & $30.03 (\pm 2.74)$     \\
    IGIWERM (bopt)     & ${\bf 28.89 (\pm 2.42)}$  \\
    IGIWERM (IC)       & ${\bf 28.38 (\pm 2.12)}$   \\
    \end{tabular}
\end{table}

Figure~\ref{fig:geometry_of_covariate_shift} shows the curves on the manifolds created by AIWERM and RIWERM.
Both of them satisfy
\begin{itemize}
    \item for $\lambda=0$, it is equivalent to unweighted ERM,
    \item for $\lambda=1$, it is equivalent to IWERM.
\end{itemize}
Note that the curvature bias $b_c$ vanishes for all $\lambda \in [0,1]$ in AIWERM, while RIWERM does not guarantee the vanishing of the curvature bias for $\lambda\in(0,1)$.

Intuitively, the geometric bias reveals in which direction the two parameters are misaligned.
IWERM, which corresponds to AIWERM and RIWERM with $\lambda=1$, is optimal when the sample size is large enough, but in real problems with limited sample size, it is often desirable to adopt a point between $\bm{\theta}_{tr}$ and $\bm{\theta}_{te}$. AIWERM and RIWERM consider distinct curves and specify a point on them by the parameter $\lambda$. 
Our geometric analysis revealed that these curves are included in the set of curves represented by dual $f$-representation of the parameter coordinate system, and the geometric biases of these particular cases (AIWERM and RIWERM) are identified. The results presented in this subsection do not claim superiority of a particular method and are of importance in their own right as a geometric analysis of the covariate shift method. 

\subsection{Optimization of the generalized IWERM}
\label{subsec:optimization_of_the_generalized_iwerm}
The existing covariate shift adaptation methods described above can be regarded as having determined a good ``weighting direction'' in some sense in advance and then the ``weighting magnitude'' is adjusted according to the parameter $\lambda$.
This approach is very convenient in terms of computational efficiency since the only optimized parameter is $\lambda\in[0,1]$.

However, geometrically, these methods only consider certain curves on the manifold as candidate solutions, as can be seen from Figure~\ref{fig:geometry_of_covariate_shift}, which means that the solution space is very small.

Our information geometrical IWERM (IGIWERM) can handle all curves $\gamma_{\alpha}(\lambda)$ in $\Pi_{(p_{tr}, p_
{te})}$ that connect $p_{tr}(\bm{x})$ and $p_{te}(\bm{x})$, by adding only one parameter.
For example, by setting $\alpha\in[1,3]$, shaded area in Figure~\ref{fig:geometry_of_covariate_shift} can be used as the solution space. The problem of how to determine $\lambda$ and $\alpha$ remains. 

\subsubsection{Information criterion}
When the predictive model is of a simple parametric form, information criterion derived in~\citep{Shimodaira2000-vv} is available (see appendix of~\citep{Shimodaira2000-vv} for the proof.):
\begin{theo}[Information criterion for IGIWERM]
Let the information criterion for IWERM be
\begin{equation}
    IC_{GW} \coloneqq -2L_1(\hat{\bm{\theta}}) + 2tr(J_{w}H_w^{-1}), \label{eq:information_criterion}
\end{equation}
where $L_1(\bm{\theta}) = \sum^{n_{tr}}_{i=1}
dr(\bm{x}^{tr}_i) 
\log p(y^{tr}_i | \bm{x}^{tr}_i, \bm{\theta}), \;
dr(\bm{x}) = \frac{p_{te}(\bm{x})}{p_{tr}(\bm{x)}}
$ and
\begin{align*}    
    J_w &= -\mathbb{E}_{p_{tr}}\Biggl[
    dr(\bm{x})
    \frac{\partial \log p(y|\bm{x},\bm{\theta})}{\partial\bm{\theta}}\Biggr|_{\bm{\theta}^*_w} 
    \times
    \frac{\partial\Big(\frac{m_f^{\lambda, \alpha}(p_{tr}(x), p_{te}(\bm{x}))}{p_{tr}(\bm{x})}
    \log p(y | \bm{x},\bm{\theta})
    \Big)
    }{\partial\bm{\theta}'}\Biggr|_{\bm{\theta}^*_w}\Biggr] \\
    H_w &= \mathbb{E}_{p_{tr}}\Biggl[\frac{\partial^2\Big( \frac{m_f^{\lambda, \alpha}(p_{tr}(x), p_{te}(\bm{x}))}{p_{tr}(\bm{x})} \log p(y | \bm{x},\bm{\theta})\Big)}{\partial\bm{\theta}\partial\bm{\theta}'}\Biggr].
\end{align*}
Here, $\bm{\theta}^*_w$ is the minimizer of the weighted empirical risk. 
The matrices $J_w$ and $H_w$ may be replaced by their consistent estimates. 
Then, $IC_{GW}/2n$ is an unbiased estimator of the expected loss up to $O(n^{-1})$ term:
\begin{equation}
    \mathbb{E}_{p_{tr}}\Big[IC_{GW}/2n\Big] = \mathbb{E}_{p_{tr}}\Big[\ell_1(\hat{\bm{\theta}}_w)\Big] + o(n^{-1}).
\end{equation}
\end{theo}

\begin{table*}[t]
\centering
\caption{Mean misclassification rates averaged over $10$ trails on LIBSVM benchmark datasets. 
The numbers in the brackets are the standard deviations.
For the methods with (optimal), the optimal parameters for the test data are obtained by linear search.}
\label{tab:libsvm_error}
\resizebox{\textwidth}{!}{%
\begin{tabular}{c|c|c|ccccc}
Dataset       & \#features & \#data  & unweighted          & IWERM               & AIWERM (optimal)   & RIWERM (optimal)     & ours                        \\ \hline
australian    & $14$       & $690$   & $33.46 (\pm 23.65)$ & $22.13 (\pm 3.37)$  & $21.98 (\pm 3.36)$ & $21.73 (\pm 3.82)$  & ${\bf 18.85 (\pm 3.99)}$ \\
breast-cancer & $10$       & $683$   & $38.28 (\pm 10.98)$ & $41.23 (\pm 15.39)$ & $36.41 (\pm 9.68)$ & $36.13 (\pm 10.81)$ & ${\bf 31.65 (\pm 8.49)}$ \\
heart         & $13$       & $270$   & $45.17 (\pm 6.98)$  & $39.94 (\pm 8.55)$  & $39.76 (\pm 8.49)$ & $39.76 (\pm 8.92)$  & ${\bf 35.37 (\pm 6.84)}$ \\
diabetes      & $8$        & $768$   & $33.19 (\pm 5.69)$  & $37.22 (\pm 6.63)$  & $33.11 (\pm 6.45)$ & $33.38 (\pm 5.74)$  & ${\bf 32.83 (\pm 5.62)}$ \\
madelon       & $500$      & $2,000$ & $47.78 (\pm 1.53)$  & $47.28 (\pm 2.20)$  & $47.10 (\pm 2.13)$ & $47.12 (\pm 1.65)$  & ${\bf 46.56 (\pm 2.12)}$ \\
\end{tabular}%
}
\end{table*}

\subsubsection{Bayesian optimization}
This information criterion does not work for complicated nonparametric models. As a method that can be applied in general situations, we consider using Bayesian optimization~\cite{Snoek2012-rs,Frazier2018-fm} to find the optimal weighting by IGIWERM. Bayesian optimization assumes that the target function is drawn from a prior distribution over functions, typically a Gaussian process, updating a posterior as we observe the target function value in new places. 
We use the validation loss as the target function:
\begin{equation}
    L(h;\lambda,\alpha) = \frac{1}{n_{val}}\sum^{n_{val}}_{i=1}\frac{p_{te}(\bm{x}_i^{val})}{p_{tr}(\bm{x}_i^{val})}\ell(h_{\lambda,\alpha}(\bm{x}^{val}_i), y^{val}_i). \label{eq:target_function}
\end{equation}
where $n_{val}$ is the validation sample size and $h_{\lambda,\alpha}$ is given by IWERM with $\lambda$ and $\alpha$.
This validation procedure is based on the importance weighted cross validation used in~\citep{Sugiyama2007-lr}.
In Bayesian optimization, an acquisition function $a(\lambda, \alpha |D)$ is used for measuring goodness of candidate point $(\lambda, \alpha)$ based on current dataset $D$. 
As the acquisition function, we adopt the expected improvement~\cite{Mockus1978-dh,Jones1998-xu}.
In this strategy, we choose the next query point which has the highest expected improvement over the current minimum target value. See Appendix~\ref{app:BO} for more detail. The overall picture is summarized in Algorithm~\ref{alg:bopt}.

\subsection{Learning guarantee}
Generalization bounds of weighted maximum likelihood estimator for the target domain are derived in~\cite{Cortes2010}, and our weight function~\eqref{eq:Gweight} is compatible with their bound. The weight defined by Eq.~\eqref{eq:Gweight} is bounded when $\alpha \neq 1$ and achieves a standard rate $O(n_{tr}^{-1/2})$. When $\alpha=1$, the weight is unbounded and its rate is $O(n_{tr}^{-3/8})$. Details are shown in Appendix~\ref{app:slt}.

\section{Numerical Experiments}
\label{sec:numerical_experiments}
In this section, we present experimental results of domain adaptation problems under covariate shift using both synthetic and real data\footnote{Source code to reproduce the results is available from~ \url{https://github.com/nocotan/IGIWERM}}.
Since the main purpose of the experiments is to see the effect of our generalization of the importance weighted ERM and comparison to the proposed and conventional IWERM methods, in all experiments, we assume that $p_{tr}$ and $p_{te}$ are known as detailed in Section~\ref{subsec:inducing_covariate_shift}.

\subsection{Induction of Covariate Shift}
\label{subsec:inducing_covariate_shift}

Since each dataset is composed of data points generated from independent and identical distributions, we need to artificially induce covariate shifts.
We induce the covariate shift as follows~\cite{Cortes2008-zd}:
\begin{enumerate}
    \item As a preprocessing step, we perform Z-score standardization on all input data.
    \item Then, an example $(\bm{x}, y)$ is assigned to the training dataset with probability $\exp(v)/(1 + \exp(v))$ and to the test dataset with probability $1/(1 + \exp(v))$, where $v = 16\bm{w}^T\bm{x}/\sigma$ with $\sigma$ being the standard deviation of $\bm{w}^T\bm{x}$ determined by using the given dataset, and $\bm{w}\in\mathbb{R}^d$ is a given projection vector. Here, the projection vector $\bm{w}$ is given randomly for each experimental process.
\end{enumerate}
By this construction of the training and test datasets, $p_{tr}$ and $p_{te}$ are explicitly determined as
\begin{align*}
    p_{tr}(\bm{x}) &= \frac{\exp(16\bm{w}^T\bm{x}/\sigma)}{1 + \exp(16\bm{w}^T\bm{x}/\sigma)}, \\
    p_{te}(\bm{x}) &= \frac{1}{1 + \exp(16\bm{w}^T\bm{x}/\sigma)},
\end{align*}
when the projection vector $\bm{w}\in\mathbb{R}^d$ is given.
Although density ratio estimation could be employed in our experiments, we assume that the distribution is known in order to compare the performance of the proposed method without relying on the accuracy of the density or density ratio estimation.

\subsection{Illustrative Example in Regression}
\label{subsec:dummy_experiment}
Here, we predict the response $y\in\mathbb{R}$ using ordinary linear regression:
    $y = \beta_0 + \beta_1x + \varepsilon,\ \varepsilon\sim\mathcal{N}(0,\sigma^2)$,
where $\mathcal{N}(a,b)$ denotes the normal distribution with mean $a$ and variance $b$.
In the numerical example below, we assume the true $p(y|x)$ given as $
    y = x^2 + \varepsilon,\ \varepsilon\sim\mathcal{N}(0,5).
$
The $p_{tr}(x)$ and $p_{te}(x)$ of the covariate $x$ are $
    x^{tr}\sim\mathcal{N}(0,5),\ x^{te}\sim\mathcal{N}(-5,0.5).
$
The training sample size is $n_{tr}=1000$ and the test sample size is $n_{te}=300$.
The left-hand side of Fig.~\ref{fig:dummy_experiment} shows the data to be generated.
We can see that $p_{tr}(x)\neq p_{te}(x)$.

The right panel of Fig.~\ref{fig:dummy_experiment} shows the results of fitting by unweighted ERM, IWERM, and IGIWERM.
Here, the parameters of IGIWERM are explored by using  Algorithm~\ref{alg:bopt}, as shown in Fig.~\ref{fig:bopt}.
The coordinates of the purple circles are the parameters explored by Bayesian optimization, and the radius of the purple circles is proportional to the goodness $r(\bm{\beta})$ of the parameters (inverse of the MSE):
    $r(\bm{\beta}) = \left(\frac{1}{n}\sum^n_{i=1}(y_i - h(x_i, \bm{\beta}))^2\right)^{-1}$.
By choosing the size $r(\beta)$ of the plot for each point in this manner, the better-evaluated parameters can be plotted in larger circles.
From this figure, it can be seen that our generalized weighting is not restricted to lying just on two curves corresponding to AIWERM and RIWERM. 

For the normal linear regression, the information criterion~\eqref{eq:information_criterion} is calculated from
\begin{align*}
&IC_{GW}(\lambda,\alpha) = \frac{1}{2}\sum^{n_{tr}}_{i=1}
dr(\bm{x}_{i}^{tr})
\Biggl\{\frac{\hat{\epsilon}^2_1}{\hat{\sigma}^2 + \log(2\pi\hat{\sigma}^2)}\Biggr\}+\\
    & \sum^{n_{tr}}_{i=1}
    dr(\bm{x}_{i}^{tr})
    \Biggl\{\frac{\hat{\epsilon}^2_i}{\hat{\sigma}^2}\hat{h}_i + \frac{m_f^{\lambda, \alpha}(p_{tr}(\bm{x}^{tr}), p_{te}(\bm{x}^{tr}))}{2\hat{c}_w p_{tr}(\bm{x}^{tr})}\Big(\frac{\hat{\epsilon}^2_i}{\hat{\sigma}^2} - 1\Big)^2\Biggr\}.
\end{align*}
Here, $\hat{c}_w=\sum^n_{i=1}\frac{m_f^{\lambda, \alpha}(p_{tr}(x_i^{tr}), p_{te}(x_i^{tr}))}{p_{tr}(x_i^{tr})}$, $\hat{\sigma}^2=\sum^n_{i=1}\frac{m_f^{\lambda, \alpha}(p_{tr}(x_i^{tr}), p_{te}(x_i^{tr}))}{p_{tr}(x_i^{tr})}\hat{\epsilon}_i^2/\hat{c}_w$ and $\hat{\epsilon}_i$ is the residual.
Table~\ref{tab:dummy_mse} shows that the IGIWERM outperforms existing methods.

\subsection{Experiments on binary classification problem}
We show the results of our experiments on the LIBSVM dataset\footnote{\href{https://www.csie.ntu.edu.tw/~cjlin/libsvmtools/datasets/}{https://www.csie.ntu.edu.tw/~cjlin/libsvmtools/datasets/}}.

In the experiments, we randomly generate a mapping vector $\bm{w}$ for each trial and perform $10$ trials for each dataset.
We use SVM with Radial Basis Function (RBF) kernel as the base classifier. In this experiment, the parameters $\lambda$ of AIWERM and RIWERM are chosen optimally by linear search using the test data. The experimental results on benchmark datasets are summarized in Table~\ref{tab:libsvm_error}.
The table shows that the proposed IGIWERM outperforms the conventional methods even when the parameters of those methods are optimized by using the test dataset.
More experimental results on other datasets with various models are reported in Appendix~\ref{app:exp}.

\subsection{Computational Cost}
Here, we investigate the computational cost of our IGIWERM.
The experimental setup is the same as in Section~\ref{subsec:dummy_experiment}.
The mean and standard deviation of the computation time obtained in the 10 trials are shown in Table~\ref{tab:computational_cost}.
From this table, we can see that our IGIWERM takes constant times longer to compute than the vanilla ERM.

\begin{table}[th]
    \centering
    \caption{Computational cost of ERM and IGIWERM.}
    \begin{tabular}{c|c}
         Method  & Computation time [sec] \\ \hline
         ERM     & $1.130(\pm 0.238)$     \\
         IGIWERM & $9.887(\pm 0.845)$     \\
    \end{tabular}
    \label{tab:computational_cost}
\end{table}

\section{Conclusion and Discussion}
We generalized existing methods of covariate shift adaptation in the geometrical framework. By our information geometrical formulation, geometric biases of conventional methods are elucidated. 
Unlike the dominant approaches restricted to a specific curve on a manifold in the literature, our generalization has a much larger solution space with only two parameters.
Our experiments highlighted the advantage of our method over previous approaches, suggesting that our generalization can achieve better performance than the existing methods. A drawback of our proposed method is its relatively high computational cost for optimizing parameters $\alpha$ and $\lambda$. We used Bayesian optimization for efficient parameter search, and further sophisticated approaches would be explored in our future work.

As mentioned in the introduction, the importance weighting is used with deep neural network models~\citep{DBLP:conf/nips/FangL0S20}, in which the importance weight in the feature representation obtained by DNN is considered. It is also worth mentioning that \citet{DBLP:conf/aaai/SakaiS19} used RIWERM in the study of covariate shift on the learning from positive and unlabeled data. 
Our generalization will be applicable to their methods to improve the performance under a small sample regime. In particular, in a standard approach for optimizing the implicit weight function $w(\bm{x})$, it is common to add a regularization term $(w(\bm{x})- p_{tr}(\bm{x}) / p_{te}(\bm{x}))^2$ to the optimization objective. The use of the derived geodesic and curvature biases to regularize the optimal weight function will be investigated in connection with the modern weight learning approach using deep neural network models. Finally, the relation between geometric bias and statistical bias should be explored.

\subsection*{Acknowledgement}
Part of this work is supported by JST CREST JPMJCR1761, JPMJCR2015, JSPS KAKENHI 17H01793, JP22H03653 and NEDO (JPNP18002). Finally, we express our special thanks to the anonymous reviewers whose valuable comments helped to improve the manuscript.

\section*{Appendix A: Statistical Manifolds and Straight Line in Exponential Family}
\label{app:mfd}
Let $M$ be a $d$-dimensional differentiable manifold with a Riemannian metric $g$. For each $\bm{x} \in M$, $T_{\bm{x}} M$ is its tangent space. 

\begin{defi}
Let $g_{\bm{x}}$ an inner product 
\begin{equation}
    g_{\bm{x}} : T_{\bm{x}}(M) \times T_{\bm{x}}(M) \to \mathbb{R} \quad \forall \bm{x} \in M.
\end{equation}
When, for any $X, Y \in M$, the map $\bm{x} \to g_{\bm{x}}(X_{\bm{x}},Y_{\bm{x}})$ is differentiable with respect to $\bm{x} \in M$, $g_{\bm{x}}$ is denoted as the Riemannian metric.
\end{defi}

The correspondence $X: M \ni \bm{x} \mapsto X_{\bm{x}} \in T_{\bm{x}}M$ is called a vector field on $M$. For $\bm{x} \in M$, let  coordinate expression of $X_{\bm{x}}$ be $X_{\bm{x}} = (v^1(\bm{x}), \dots,v^{d}(\bm{x}))$. Then, $v^{i}(\bm{x})$ defines a real-valued function $v^{i}$ on $M$ and $X$ is expressed as $X=(v^1,\dots,v^d)$. 
When a function on $M$ is $k$ times continuously differentiable, it is called the class $C^k$, and the set of all functions of class $C^{k}$ on $M$ is denoted as $C^{k}(M)$.
A vector field $X$ is called class $C^{k}$ when all of $v^i, \; i=1,\dots,d$ are class $C^{k}$. The set of all class $C^{\infty}$ vector fields is denoted as $\mathfrak{X}(M)$. A tangent space $T_{\bm{x}}(M)$ is a vector space spanned by differentials $\frac{\partial}{\partial x^i}$, namely,
\begin{equation}
    T_{\bm{x}}(M) = 
    \left\{
    a^i
    \left(
    \frac{\partial}{\partial x^{i}}
    \right)_{\bm{x}}
    \middle|
    \forall \; a_i  \in \mathbb{R}
    \right\}.
\end{equation}
Following the notational convention of differential geometry, we use $\partial_i = \frac{\partial}{\partial x^{i}}$ and the Einstein summation convention. The vector field on a manifold $M$ is then written as
\begin{equation}
    \mathfrak{X}(M) = 
    \left\{
    v^{i} \partial_{i} 
    \middle|
    \; v^i \in C^{\infty}(M)
    \right\}.
\end{equation}
For $X \in \mathfrak{X}(M)$ and $f \in C^{\infty}(M)$, $fX \in \mathfrak{X}(M)$ is defined by $(fX)_{\bm{x}} = f(\bm{x}) X_{\bm{x}}, \; (\bm{x} \in M)$. Differential of a function $f$ with respect to a vector field $X$ is denoted as $Xf \in C^{\infty}(M)$ and defined by $(Xf)(\bm{x}) = X_{\bm{x}}(f), \; (\bm{x} \in M)$. When two vector fields are expressed as $X=v^i \partial_{i}$ and $Y=u^i \partial_{i}$, we have 
\begin{equation}
X(Yf) - Y(Xf) = (v^j \partial_j u^i - u^j \partial_j v^i) \partial_i f.
\end{equation}
The commutator product of $X$ and $Y$ is defined as $[X,Y] \in \mathfrak{X}(M), \; [X,Y]f = (XY-YX)f$, and 
\begin{equation}
[X,Y] =  (v^j \partial_j u^i - u^j \partial_j v^i) \partial_i.
\end{equation}

\begin{defi}
Consider a map $\nabla: \mathfrak{X}(M) \times \mathfrak{X}(D) \to \mathfrak{X}(M)$ which assigns a pair of vectors $(X,Y) \in \mathfrak{X}(M)\times \mathfrak{X}(M)$ to a vector $\nabla_{Y}X \in \mathfrak{X}(M)$. $\nabla_{Y}X$ is called a covariant derivative of $X$ with respect to $Y$, and $\nabla$ is called an affine connection when the following conditions hold for any $X,Y,Z \in \mathfrak{X}(M)$ and $f \in C^{\infty}(M)$:
\begin{itemize}
    \item $\nabla_{Y+Z}X=\nabla_{Y}X + \nabla_{Z}X$
    \item $\nabla_{fX}X = f \nabla_{Y}X$
    \item $\nabla_{Z}(X+Y) = \nabla_{Z}X + \nabla_{Z}Y$
    \item $\nabla_{Y}(fX) = (Yf)X + f\nabla_Y X$
\end{itemize}
\end{defi}

\begin{defi}
Let $\nabla$ be an affine connection on $M$, and define a map
\begin{align}
\notag
    T : \mathfrak{X}(M) \times \mathfrak{X}(M) \to & \mathfrak{X}(M)\\
    (X,Y)  \mapsto & T(X,Y) = \nabla_X Y - \nabla_Y X - [X,Y].
\end{align}
The map $T$ is called the torsion tensor field of $\nabla$. When $T=0$ for all $X,Y \in \mathfrak{X}(M)$, the connection $\nabla$ is called torsion-free.
\end{defi}

For an affine connection, the Christoffel symbol $\Gamma^{k}_{ij} \in C^{\infty}(M)$ is defined by
\begin{equation}
    \nabla_{\partial_i} \partial_j = 
    \Gamma^{k}_{ij} \partial_k.
\end{equation}
With this formula, the connection and the Christoffel symbol are often identified. 
The affine connection $\nabla$ is torsion-free when and only when $\Gamma^k_{ij} = \Gamma^k_{ji}$. 

Suppose a manifold $M$ is equipped with a Riemannian metric $g$. When
\begin{equation}
    X g(Y,Z) = 
    g( \nabla_{X} Y,Z)
    +
    g(Y,\nabla_{X} Z)
\end{equation}
holds for all $X,Y,Z \in \mathfrak{X}(M)$, the connection $\nabla$ is called a metric connection. In general, an affine connection is not a metric connection, but there uniquely exists an affine connection $\nabla^{\ast}$ which satisfies
\begin{equation}
    X g(Y,Z) = 
    g(\nabla_X Y,Z) +
    g(Y,\nabla^{\ast}_X Z).
\end{equation}
The connection $\nabla^{\ast}$ is called the dual connection of $\nabla$.

Given a Riemannian metric $g$, another reperesentation of the Christoffel symbol is given by
\begin{equation}
    \Gamma_{ij,k} = 
    g
    \left(
    \nabla_{\partial_i} \partial_j, \partial_k
    \right).
\end{equation}

\begin{defi}
When an affine connection $\nabla$ is torsion-free and a metric connection with respect to the Riemannian metric $g$, it is called a Levi-Civita connection with respect to the metric $g$.
\end{defi}

In general, when a $(0,3)$-tensor $\bar{T}$ is given in addition to an affine connection $\nabla$ and a Riemannian metric $g$, an alternative connection $\tilde{\nabla}$ is defined as 
\begin{equation}
    g(\tilde{\nabla}_{Y} X,Z) =
    g(\nabla_{Y}X,Z) + 
    \bar{T}(X,Y,Z).
\end{equation}

Let $\Omega$ be a set for which probability measure is defined, and define a $d$-dimensional statistical model 
\begin{equation}
    S = \{
    p(\cdot ; \bm{\xi}) | \bm{\xi} \in \Xi
    \},
\end{equation}
where the parameter space $\Xi$ is isomorphic to $\mathbb{R}^d$. As a Riemannian metric associated with the statistical model $S$, we consider the Fisher metric defined as
\begin{equation}
    g_{ij}(\bm{\xi}) = 
    \mathbb{E}_{\bm{\xi}}
    [
    (\partial_i l_{\bm{\xi}})
    (\partial_j l_{\bm{\xi}})
    ],
\end{equation}
where $\mathbb{E}_{\bm{\xi}} [\cdot]$ is expectation with respect to a probability density $p(\cdot; \bm{\xi})$ and $l_{\bm{\xi}}(x) = \log p(x;\bm{\xi})  \; (x \in \Omega)$ is the log-likelihood. 
Now, consider a $(0,3)$-tensor $\bar{T}$ on $S$ defined by
\begin{equation}
    (\bar{T})_{ijk} (\bm{\xi})
    =
    \sum_{x \in \Omega}
    (\partial_{i} l_{\bm{\xi}}(x))
    (\partial_{j} l_{\bm{\xi}}(x))
    (\partial_{k} l_{\bm{\xi}}(x))
    p(x;\bm{\xi}),
\end{equation}
and based on the Levi-Civita connection $\nabla$ associated with the Fisher metric $g$ on $S$, we define a affine connection $\nabla^{(\alpha)}$ by
\begin{equation}
    g(\nabla^{(\alpha)}_{Y}X,Z) =
    g(\nabla_Y X,Z) - 
    \frac{\alpha}{2}
    \bar{T}(X,Y,Z), \quad 
    (X,Y,Z \in \mathfrak{X}(S)).
\end{equation}
This connection is called the $\alpha$-connection. The Christoffel symbols associated with connections $\nabla$ and $\nabla^{(\alpha)}$ are 
\begin{align*}
    \Gamma_{ij,k} =&
    \mathbb{E}_{\bm{\xi}}
    \left[
    \left\{
    \partial_i \partial_j l_{\bm{\xi}} + 
    \frac{1}{2} (\partial_i l_{\bm{\xi}})
    (\partial_j l_{\bm{\xi}})
    \right\}
    (\partial_k l_{\bm{\xi}})
    \right],\\
     \Gamma_{ij,k}^{(\alpha)} =&
    \mathbb{E}_{\bm{\xi}}
    \left[
    \left\{
    \partial_i \partial_j l_{\bm{\xi}} + 
    \frac{1-\alpha}{2} (\partial_i l_{\bm{\xi}})
    (\partial_j l_{\bm{\xi}})
    \right\}
    (\partial_k l_{\bm{\xi}})
    \right].
\end{align*}
From $\Gamma_{ij,k}^{(\alpha)} = \Gamma_{ji,k}^{(\alpha)}
$, the $\alpha$-connection is torsion-free. 
Note that the dual connection of $\nabla^{(\alpha)}$ is $\nabla^{(-\alpha)}$, and it also holds that 
\begin{equation}
    \nabla^{(\alpha)} = 
    \frac{1+\alpha}{2} \nabla^{\ast} + 
    \frac{1-\alpha}{2} \nabla.
\end{equation}

\begin{defi}
\label{def:curvature}
For an affine connection $\nabla$ of a manifold $M$, a map
\begin{align*}
    R : \mathfrak{X}(M) \times \mathfrak{X}(M)
    \times \mathfrak{X}(M)
    \to & \mathfrak{X}(D)\\
    (X,Y,Z) \mapsto & R(X,Y)Z = 
    \nabla_X \nabla_Y Z 
    -
    \nabla_Y \nabla_X Z
    -
    \nabla_{[X,Y]}Z
\end{align*}
is called the curvature tensor field of the connection $\nabla$.
\end{defi}
The curvature tensor is expressed with coordinate and the Christoffel symbol as
\begin{equation}
    R(\partial_i,\partial_j)\partial_k 
    =
    (\partial_i \Gamma_{jk}^l 
    -
    \partial_j 
    \Gamma^l_{ik}) \partial_l 
    +
    (
    \Gamma^l_{jk} \Gamma^{m}_{il}
    -
    \Gamma^{l}_{ik} 
    \Gamma^{m}_{jl}
    )
    \partial_m.
\end{equation}

\begin{defi}
When both the torsion and curvature are zero, the connection $\nabla$ is said to be flat.
\end{defi}

Let $\gamma$ be a map from a close interval $I$ to a manifold $M$. The map $\gamma$ is parameterized by a real-valued parameter $t \in I$ as $\gamma(t)$ and called a curve on $M$. When the value of $\gamma$ at two endpoints of $I$ is fixed, the shortest path between these two points is defined by using the variational principle. 
The pararell shift of $\frac{d \gamma}{d t}$ along with $\gamma$ is expressed as
\begin{equation}
    \nabla_{\frac{d}{dt}}^{\gamma} \frac{d \gamma}{dt}
    =
    \left(
    \frac{d^2 \gamma_k}{d t^2}
    +
    (\Gamma^k_{ij} \circ \gamma) 
    \frac{d \gamma_i}{dt}
    \frac{d \gamma_j}{dt}
    \right) \partial_k.
\end{equation}

\begin{defi}
An equation
\begin{equation}
    \nabla_{\frac{d}{dt}}^{\gamma}\frac{d \gamma}{dt} = \bm{0}
\end{equation}
is called the geodesic equation, and the curve satisfying this equation is called a geodesic.
\end{defi}

Note that if $\Gamma_{ij,k} = 0 \; \forall i,j,k$, the geodesic equation is of the form $\frac{d^2 \gamma_{k}}{d t^2} = 0$, hence the geodesic is a straight line.

\begin{defi}
Let $S$ be a $d$-dimensional statistical model. When each element of the model in $S$ is represented by 
\begin{equation}
    p(x;\bm{\theta}) = 
    \exp
    \left(
    k(x) + 
    \theta^i F_i (x) - \psi(\bm{\theta})
    \right),
\end{equation}
by using functions $k,F_1,\dots,F_d : \Omega \to \mathbb{R}$ and $\psi : \Theta \to \mathbb{R}$, the statistical model $S$ is called an exponential family, and $\bm{\theta}$ is called the natural parameter of the model. 
\end{defi}
Note that in a general statistical model $S$, $\xi$, and $\Xi$ are often used as its parameter and the parameter space, while for an exponential family, $\theta$ and $\Theta$ are often used to represent its parameter and the parameter space.
Consider an exponential family with $\alpha$ connection $\nabla^{(\alpha)}$. The Christoffel symbols are
\begin{equation}
    \Gamma^{(\alpha)}_{ij,k}
    =
    \mathbb{E}_{\bm{\theta}}
\left[
\left\{
\partial_i \partial_j l_{\bm{\theta}} + 
\frac{1-\alpha}{2}
(\partial_i l_{\bm{\theta}})
(\partial_j l_{\bm{\theta}})
\right\}
(\partial_k l_{\bm{\theta}})
\right],
\end{equation}
and 
\begin{equation}
    \partial_i l_{\bm{\theta}} = 
    F_i (x) - (\partial_i \psi)(\bm{\theta}), \quad 
    (\partial_i \partial_j \psi) (\bm{\theta}).
\end{equation}
So, when $\alpha=1$, we have
\begin{equation}
\Gamma^{(1)}_{ij,k}
=
\mathbb{E}_{\bm{\theta}} 
[
(- ( \partial_i \partial_j \psi)(\bm{\theta}))
(\partial_k l_{\bm{\theta}})
] = 0,
\end{equation}
namely, the exponential family is flat with the Fisher metric and $\alpha=1$ connection. 
This implies that in exponential family, for the $\alpha=1$-connection $\nabla^{(1)}$ associated with the Fisher metric, the geodesic between two points correspond to natural parameters $\bm{\theta}_1$ and $\bm{\theta}_{2}$ is of the form $t \bm{\theta}_1 + (1-t) \bm{\theta}_2$.


\section*{Appendix B: Proofs of main results}
\label{app:prf}
\begin{proof}[Derivation of the information geometrically generalized IWERM]
Let $h_A$ be a hypothesis generated by AIWERM.
From Lemma 4.1, we can write
\begin{align}
    \hat{h}_A &= \min_{h\in\mathcal{H}}\int_{\mathcal{X}\times\mathcal{Y}}\ell(h(\bm{x}), y)p^{(\lambda)}_A(\bm{x})p_{tr}(y|\bm{x})d\bm{x}dy \nonumber \\
    &= \min_{h\in\mathcal{H}}\int_{\mathcal{X}\times\mathcal{Y}}\ell(h(\bm{x}), y)m_f^{(\lambda, 1)}(p_{tr}(\bm{x}), p_{te}(\bm{x}))p_{tr}(y|\bm{x})d\bm{x}dy \nonumber \\
    &= \min_{h\in\mathcal{H}}\int_{\mathcal{X}\times\mathcal{Y}}\ell(h(\bm{x}), y)\frac{m_f^{(\lambda, 1)}(p_{tr}(\bm{x}), p_{te}(\bm{x}))}{p_{tr}(\bm{x})}p_{tr}(\bm{x}, y)d\bm{x}dy.
\end{align}
From Lemma 4.2, we also have
\begin{align}
    \hat{h}_R = \min_{h\in\mathcal{H}}\int_{\mathcal{X}\times\mathcal{Y}}\ell(h(\bm{x}), y)\frac{m_f^{(\lambda, 3)}(p_{tr}(\bm{x}), p_{te}(\bm{x}))}{p_{tr}(\bm{x})}p_{tr}(\bm{x}, y)d\bm{x}dy.
\end{align}
Then, we consider
\begin{equation}
    \hat{h} = \min_{h\in\mathcal{H}}\int_{\mathcal{X}\times\mathcal{Y}}w^{(\lambda,\alpha)}(\bm{x})\ell(h(\bm{x}), y)p_{tr}(\bm{x},y)d\bm{x}dy, 
\end{equation}
where
\begin{equation}
    w^{(\lambda,\alpha)}(\bm{x}) = 
    \frac{
    m^{(\lambda,\alpha)}_{f}(p_{tr}(\bm{x}),p_{te}(\bm{x}))
    }{p_{tr}(\bm{x})}.
\end{equation}
We can see that AIWERM is a special case when $\alpha=1$ and RIWERM is a special case when $\alpha=3$.
\end{proof}

\begin{proof}[Proofs of Propositions~\ref{prop:geometric_bias_aiwerm} and ~\ref{prop:geometric_bias_riwerm}]
Let
\begin{equation}
    \bm{\theta}^{(\lambda,\alpha)} = m^{(\lambda, \alpha)}_f(\bm{\theta}_{tr}, \bm{\theta}_{te}),
\end{equation}
and let $R^{(\alpha)}$ be the Riemann curvature tensor defined in Definition~\ref{def:curvature} with respect to the $\alpha$-connection $\nabla^{(\alpha)}$.

We define the relative curvature tensor as
\begin{equation}
R^{(\alpha,\beta)}(X,Y,Z) = \Big[\nabla^{(\alpha)}_X, \nabla^{(\beta)}_Y\Big]Z - \nabla^{(\alpha)}_{[X,Y]}Z
\end{equation}
and
the difference tensor as
\begin{equation}
K(X,Y) = \nabla^*_XY - \nabla_XY.
\end{equation}

For any $\alpha\in\mathbb{R}$ and $\beta\in\mathbb{R}$, we have
\begin{align}
    \nabla^{(\alpha)}_X\nabla^{(\beta)}_Y Z &= \Big(\frac{1+\alpha}{2}\nabla^*_X + \frac{1-\alpha}{2}\nabla_X\Big)\Big(\frac{1+\beta}{2}\nabla^*_Y + \frac{1-\beta}{2}\nabla_Y\Big)Z \nonumber\\
    &= \frac{(1+\alpha)(1+\beta)}{4}\nabla^*_X\nabla^*_YZ + \frac{(1+\alpha)(1-\beta)}{4}\nabla^*_X\nabla_YZ \nonumber\\
    &\ \quad + \frac{(1-\alpha)(1+\beta)}{4}\nabla_X\nabla^*_YZ+\frac{(1-\alpha)(1-\beta)}{4}\nabla_X\nabla_YZ. \\
    \nabla^{(\beta)}_Y\nabla^{(\alpha)}_X &= \Big(\frac{1+\beta}{2}\nabla^*_Y + \frac{1-\beta}{2}\nabla_Y\Big)\Big(\frac{1+\alpha}{2}\nabla^*_X + \frac{1-\alpha}{2}\nabla_X\Big)Z \nonumber \\
    &= \frac{(1+\alpha)(1+\beta)}{4}\nabla^*_Y\nabla^*_XZ + \frac{(1-\alpha)(1+\beta)}{4}\nabla^*_Y\nabla_XZ \nonumber\\
    &\ \quad + \frac{(1+\alpha)(1-\beta)}{4}\nabla_Y\nabla^*_XZ + \frac{(1-\alpha)(1-\beta)}{4}\nabla_Y\nabla_XZ.\\
    \nabla^{(\alpha)}_{[X,Y]}Z &= \frac{1+\alpha}{2}\nabla^*_{[X,Y]}Z + \frac{1-\alpha}{2}\nabla_{[X,Y]}Z.
\end{align}
Then
\begin{align}
    R^{(\alpha,\beta)}(X,Y,Z) &= \nabla^{(\alpha)}_X\nabla^{(\beta)}_Y Z - \nabla^{(\beta)}_X\nabla^{(\alpha)}_Y Z - \nabla^{(\alpha)}_{[X,Y]} Z\nonumber \\
    &= \frac{(1+\alpha)(1+\beta)}{4}(\nabla^*_X\nabla^*_Y - \nabla^*_Y\nabla^*_X)Z \nonumber\\
    &\ \quad + \frac{(1+\alpha)(1-\beta)}{4}(\nabla^*_X\nabla_Y - \nabla_Y\nabla^*_X)Z \nonumber\\
    &\ \quad + \frac{(1-\alpha)(1+\beta)}{4}(\nabla_X\nabla^*_Y - \nabla^*_Y\nabla_X)Z \nonumber\\
    &\ \quad + \frac{(1-\alpha)(1-\beta)}{4}(\nabla_X\nabla_Y - \nabla_Y\nabla_X)Z \nonumber\\
    &\ \quad -\frac{1+\alpha}{2}\nabla^*_{[X,Y]}Z - \frac{1-\alpha}{2}\nabla_{[X,Y]}Z \nonumber\\
    &= \frac{(1+\alpha)(1+\beta)}{4}\Big\{R^*(X,Y,Z) + \nabla^*_{[X,Y]}Z\Big\} \nonumber\\
    &\ \quad + \frac{(1+\alpha)(1-\beta)}{4}\Big\{R^{(1,-1)}(X,Y,Z)+\nabla^*_{[X,Y]}Z\Big\} \nonumber\\
    &\ \quad + \frac{(1-\alpha)(1+\beta)}{4}\Big\{R^{(-1,1)}(X,Y,Z) + \nabla_{[X,Y]}Z\Big\} \nonumber\\
    &\ \quad + \frac{(1-\alpha)(1-\beta)}{4}\Big\{R^{(-1,-1)}(X,Y,Z) + \nabla^*_{[X,Y]}Z\Big\} \nonumber\\
    &\ \quad - \frac{1+\alpha}{2}\nabla^*_{[X,Y]}Z - \frac{1-\alpha}{2}\nabla_{[X,Y]}Z \\ 
    4R^{(\alpha,\beta)} &= (1+\alpha)(1+\beta)R^* + (1-\alpha)(1-\beta)R \nonumber\\
    &\ \quad + (1+\alpha)(1-\beta)R^{(1,-1)} + (1-\alpha)(1+\beta)R^{(-1,1)}.
\end{align}
We also have
\begin{align}
    K^{(\alpha,\beta)}(X,Y) &= \nabla^{(\beta)}_XY - \nabla^{(\alpha)}_XY \nonumber\\
    &= \Big\{\frac{1+\beta}{2}\nabla^*_XY + \frac{1-\beta}{2}\nabla_XY\Big\} - \Big\{\frac{1+\alpha}{2}\nabla^*_XY + \frac{1-\alpha}{2}\nabla_XY\Big\} \nonumber\\
    &= \frac{\beta-\alpha}{2}\nabla^*_XY - \frac{\beta-\alpha}{2}\nabla_XY = \frac{\beta-\alpha}{2}K(X, Y) \\
    K^{(\alpha,\beta)}\Big(X, K^{(\alpha,\beta)}(Y,Z)\Big) &= \frac{\beta-\alpha}{2}K\Big(X, K^{(\alpha,\beta)}(Y,Z)\Big) = \frac{(\beta-\alpha)^2}{4}K\Big(X, K(Y,Z)\Big).
\end{align}
Combining them, the following relations hold:
\begin{align}
    K^{(\beta,\alpha)}\Big(X, K^{(\beta,\alpha)}(Y,Z)\Big) &= K^{(\beta,\alpha)}\Big(X, \nabla^{(\alpha)}_YZ - \nabla^{(\beta)}_YZ\Big) \\
    &= K^{(\beta,\alpha)}\Big(X, \nabla^{(\alpha)}_YZ \Big) - K^{(\beta,\alpha)}\Big(X,\nabla^{(\beta)}_YZ\Big) \nonumber \\
    &= \nabla^{(\alpha)}_X\nabla^{(\alpha)}_YZ - \nabla^{(\beta)}_X\nabla^{(\alpha)}_YZ - \nabla^{(\alpha)}_X\nabla^{(\beta)}_YZ + \nabla^{(\beta)}_X\nabla^{(\beta)}_YZ \\
    \frac{(\alpha-\beta)^2}{4}K\Big(X, K(Y,Z)\Big) &= \nabla^{(\alpha)}_X\nabla^{(\alpha)}_YZ - \nabla^{(\beta)}_X\nabla^{(\alpha)}_YZ - \nabla^{(\alpha)}_X\nabla^{(\beta)}_YZ + \nabla^{(\beta)}_X\nabla^{(\beta)}_YZ.
\end{align}
Swapping $X$ and $Y$, we have
\begin{align}
    &\ \frac{(\alpha-\beta)^2}{4}\Biggl\{K\Big(X, K(Y,Z)\Big) - K\Big(Y, K(X,Z)\Big)\Biggr\} \nonumber\\
    &\ = R^{(\alpha)}(X,Y,Z) + R^{(\beta)}(X,Y,Z) - \Big\{\Big[\nabla^{(\alpha)}_X, \nabla^{(\beta)}_Y\Big]Z - \nabla^{(\alpha)}_{[X,Y]}Z\Big\} \nonumber\\
    &\ \quad - \Big\{\Big[\nabla^{(\beta)}_X,\nabla^{(\alpha)}_Y\Big]Z - \nabla^{(\beta)}_{[X,Y]}Z\Big\} \nonumber \\
    &= R^{(\alpha)}(X,Y,Z) + R^{(\beta)}(X,Y,Z) - R^{(\alpha,\beta)}(X,Y,Z) - R^{(\beta,\alpha)}(X,Y,Z).
\end{align}
Making $\alpha=\beta$, we have
\begin{align}
    4R^{(\alpha)} &= (1+\alpha)^2R^* + (1-\alpha)^2R + (1-\alpha^2)R^{(1,-1)} + (1-\alpha^2)R^{(-1,1)} \nonumber\\
    &= (1+\alpha^2)R^* + (1-\alpha)^2R + (1-\alpha^2)\Big(R^{(1,-1)} + R^{(-1,1)}\Big). \label{eq:relative_curvature1}
\end{align}
Making $\alpha=1$ and $\beta=-1$, we also have
\begin{align}
    R^{(1,-1)}(X,Y,Z) + R^{(-1,1)}(X,Y,Z) &= R^*(X,Y,Z) + R(X,Y,Z) \nonumber \\
    &\ \quad - \Biggl\{K\Big(X, K(Y,Z)\Big) - K\Big(Y, K(X,Z)\Big)\Biggr\}. \label{eq:relative_curvature2}
\end{align}
From Eq.~\eqref{eq:relative_curvature1} and \eqref{eq:relative_curvature2}, we obtain
\begin{align}
    4R^{(\alpha)} &= (1+\alpha)^2R^*(X,Y,Z) + (1-\alpha)^2R(X,Y,Z) \nonumber\\
    &\ \quad + (1-\alpha^2)R^*(X,Y,Z) + (1-\alpha^2)\Biggl\{K\Big(X, K(Y,Z)\Big) - K\Big(Y, K(X,Z)\Big)\Biggr\} \nonumber \\
    &= 2(1+\alpha)R^*(X,Y,Z) + 2(1-\alpha)R(X,Y,Z) \nonumber\\
    &\ \quad + (1-\alpha^2)\Biggl\{K\Big(Y, K(X,Z)\Big) - K\Big(X, K(Y,Z)\Big)\Biggl\}
\end{align}
Since the exponential family is dually flat, that is $R=0$ and $R^*=0$, and the Riemann curvature tensor with respect to $\nabla^{(\alpha)}$ is
\begin{align}
    R^{(\alpha)}(X,Y,Z) &= \frac{1-\alpha^2}{4}\Lambda, \\
    \label{eq:Lambda}
    \Lambda &= \Big(K(Y, K(X,Z)) - K(X, K(Y,Z)) \Big).
\end{align}
Then, the geometric bias vector of $\bm{\theta}^{(\lambda,\alpha)}$ is
\begin{align}
    b(\alpha, \lambda) = (1-\lambda)\Big\{e_1 + tr_g\Big(\frac{1-\alpha^2}{2}\Lambda_{ikj}d\bm{\theta}^i \otimes d\bm{\theta}^j\Big)e_2\Big\},
\end{align}
where $tr_g$ is the trace operation on the metric tensor $g$, and $\Lambda_{ikj}$ is the element of $\Lambda$ in Eq.~\eqref{eq:Lambda}. 
Since AIWERM and RIWERM are two special cases for $\alpha=1$ and $\alpha=3$, we have
\begin{align}
    b(1, \lambda) &= (1-\lambda)e_1, \\
    b(3, \lambda) &= (1-\lambda)\Big\{e_1 + tr_g\Big(-4\Lambda_{ikj}d\bm{\theta}^i \otimes d\bm{\theta}^j\Big)e_2\Big\}.
\end{align}

\end{proof}

\section*{Appendix C: Learning guarantee}
\label{app:slt}
Generalization bounds of weighted maximum likelihood estimator for the target domain are derived in~\cite{Cortes2010}, and our weight function~\eqref{eq:Gweight} is compatible with their bound. 

Then, the gap between the expected (with respect to test distribution) loss $\mathcal{R}(h)$ and empirical risk $L(h;\lambda,\alpha)$ is bounded as
\begin{align}
\notag
&
| \mathcal{R}(h) 
    - L(h;\lambda,\alpha) |
    \leq 
    \left|
    \mathbb{E}_{p_{tr}}
    \left[
    \left\{
    \frac{p_{te}(\bm{x})}{p_{tr}(\bm{x})}
    -
    w^{(\lambda,\alpha)}(\bm{x})
    \right\}
    \right]
    \right|
    \ell(h(\bm{x},y(\bm{x})))\\ \notag
    +&
    2^{5/4}
    \max 
    \left(
    \sqrt{\mathbb{E}_{p_{tr}}
    (w^{(\lambda,\alpha)}(\bm{x}))^2 
    \ell^2 (h(\bm{x},y(\bm{x})))
    },
        \sqrt{\mathbb{E}_{\hat{p}_{tr}}
    (w^{(\lambda,\alpha)}(\bm{x}))^2 
    \ell^2 (h(\bm{x},y(\bm{x}))
    }
    \right)
    \\
    \times &
    \left(
    \frac{
    p \log \frac{2 n_{tr} e}{p}
    +
    \log \frac{4}{\delta}
    }{
    n_{tr}
    }\right)^{\frac{3}{8}}.
\end{align}
In the above inequality, $p$ is the pseudo-dimension of the function class $\{w^{\lambda,\alpha}(\bm{x}) \ell(h(\bm{x}),y(\bm{x}))| h \in \mathcal{H} \}$ where $y(\bm{x})$ is the ground truth function of connecting $\bm{x}$ and $y$ as $y=y(\bm{x})$. The first term of the r.h.s. of the above inequality is the bias introduced by using $w^{\lambda,\alpha}$ instead of a standard density ratio, and the second term reflects the variance. It is worth mentioning that the term $\mathbb{E}_{p_{tr}}
    (w^{(\lambda,\alpha)}(\bm{x}))^2 
    \ell^2 (h(\bm{x},y(\bm{x})))$ is further bounded by $d_{2}(p_{te}||p_{tr}) = \int_{x\in \mathcal{X}} \frac{p_{te}^2 (\bm{x})}{p_{tr}(\bm{x})} d\bm{x}$.

\section*{Appendix D: Optimization of the generalized IWERM}
\label{app:BO}

In the expected improvement strategy, the $t+1^{th}$ point $(\lambda,\alpha)_{t+1}$ is selected according to the following equation.
\begin{small}
\begin{equation*}
    (\lambda,\alpha)_{t+1} = \argmin_{(\lambda,\alpha)}\mathbb{E}\Big[\max\Big(0, h_{t+1}(\lambda,\alpha) - L(\lambda^\dagger,\alpha^\dagger)\Big) \Big| D_t\Big],
\end{equation*}
\end{small}
where $L(\lambda^{\dagger},\alpha^{\dagger})$ is the maximum value of empirical risk that has been encountered so far, $h_{t+1}(\lambda,\alpha)$ is the posterior mean of the surrogate at the $t+1^{th}$ step and $D_t = \{(\lambda,\alpha)_i, L(\lambda_i,\alpha_i)\}^t_{i=1}$.
This equation for Gaussian process surrogate is an analytical expression:
\begin{align*}
    a_{EI}(\lambda,\alpha) &= (\mu_t(\lambda,\alpha) - L(\lambda^{\dagger},\alpha^{\dagger}))\Phi(Z) + \sigma_t(\lambda,\alpha)\phi(Z), \\
    Z &= \frac{\mu_t(\lambda,\alpha) - L(\lambda^{\dagger},\alpha^\dagger)}{\sigma_t(\lambda,\alpha)},
\end{align*}
where $\Phi(\cdot)$ and $\phi(\cdot)$ are normal cumulative and density functions, respectively, and $\mu_t$ and $\sigma_t$ are mean and standard deviation of $\{(\lambda,\alpha)_i\}^t_{i=1}$.

\section*{Appendix E: Additional experimental results}
\label{app:exp}
\subsection{Experimental results on LIBSVM dataset}
We show that for the LIBSVM dataset, IGIWERM is effective even for multiple models.
Table~\ref{tab:libsvm_several_models_error} shows the results for each model.
We use the scikit-learn~\cite{scikit-learn} implementation of the models, and the hyperparameters of each model are the default values of this library.
Figure~\ref{fig:gird_search_surface} also shows the relationship between the two parameters of IGIWERM and the errors that can be achieved.
For this visualization, we explore the parameter pairs by grid search and evaluate their performance at that time. From this figure, it is seen that the best performance is often achieved when $\alpha \neq 1$ and $\alpha \neq 3$, showing the sub-optimality of conventional methods. 

\subsection{Experimental results on regression problems}
In this section, we present experimental results for the regression problem.
All the datasets used in this experiment are available in the scikit-learn~\cite{scikit-learn} dataset collection.
We use SVR with Radial Basis Function (RBF) kernel as the base regressor.
Table~\ref{tab:sklearn_regression} shows the results of this experiment.
In this experiment, we use MSE as a metric, and this table shows that IGIWERM is superior to existing covariate shift adaptation methods.
Figure~\ref{fig:gird_search_surface_regression} shows the relationship between the two parameters of IGIWERM and the mean squared errors that can be achieved.
This figure shows that, as in the case of binary classification, the optimal parameters do not necessarily match those of existing methods.

\subsection{Experimental results on multi-class classification}
We also introduce the additional experimental results for the multi-class classification problem.
All the datasets used in this experiment are available in the scikit-learn~\cite{scikit-learn} dataset collection.
We also use the scikit-learn~\cite{scikit-learn} implementation of the models, and the hyperparameters of each model are the default values of this library. We note that the number of training sample for {\tt{covtype}} is so large hence the results with SVM for this dataset are omitted. Table~\ref{tab:sklearn_multi_class} shows the experimental results, and we see that our proposed generalization outperforms existing methods.
Figure~\ref{fig:gird_search_surface_multi_class} shows the relationship between the two parameters of IGIWERM and the errors that can be achieved.
This figure also shows the sub-optimality of conventional methods.

\subsection{Visualization of covariate shift}
In Section 5, we induce the covariate shift by the method of Cortes et al~\cite{Cortes2008-zd}.
Figure~\ref{fig:covariate_shift_libsvm} shows a plot by PCA of each dataset splitted into the training set and test set.
This figure shows that we are able to induce a covariate shift by partitioning the dataset.

\begin{table*}[t]
\centering
\caption{Mean misclassification rates averaged over 10 trails on LIBSVM benchmark datasets. The numbers in the brackets are the
standard deviations. For the methods with (optimal), the optimal parameters for the test data are obtained by the linear search.
The lowest misclassification rates among five methods are shown with bold.}
\label{tab:libsvm_several_models_error}
\resizebox{\textwidth}{!}{%
\begin{tabular}{c|c|ccccc}
\toprule
Dataset      &model              &unweighted         &IWERM              &AIWERM (optimal)       &RIWERM (optimal)         &ours                    \\ \midrule
australian   &Logistic Regression&$22.76 (\pm 9.53)$ &$22.75 (\pm 7.73)$ &$22.19 (\pm 9.35)$     &$22.39 (\pm 7.76)$       &${\bf 21.47 (\pm 8.91)}$\\
             &SVM                &$33.46 (\pm 23.65)$&$22.13 (\pm 3.37)$ &$21.98 (\pm 3.36)$     &$21.73 (\pm 3.82)$       &${\bf 18.85 (\pm 3.99)}$\\
             &AdaBoost           &$10.20 (\pm 5.09)$ &$16.35 (\pm 10.67)$&$11.23 (\pm 3.49)$     &$15.47 (\pm 2.72)$       &${\bf 9.15 (\pm 3.93)}$ \\
             &Naive Bayes        &$17.04 (\pm 6.02)$ &$14.92 (\pm 5.96)$ &$16.33 (\pm 6.07)$     &$15.49 (\pm 6.07)$       &${\bf 14.74(\pm 5.84)}$ \\
             &Random Forest      &$8.89 (\pm 3.72)$  &$8.64 (\pm 3.20)$  &${\bf 8.29 (\pm 3.21)}$&$8.38 (\pm 3.47)$        &$8.61 (\pm 3.73)$       \\ \midrule
breast-cancer&Logistic Regression&$32.32 (\pm 3.08)$ &$32.87 (\pm 3.56)$ &$32.32 (\pm 3.08)$     &$32.32 (\pm 3.08)$       &${\bf 32.26 (\pm 3.04)}$\\
             &SVM                &$38.28 (\pm 10.98)$&$41.23 (\pm 15.39)$&$36.41 (\pm 9.68)$     &$36.13 (\pm 10.81)$      &${\bf 31.65 (\pm 8.49)}$\\
             &AdaBoost           &$5.09 (\pm 1.65)$  &$5.63 (\pm 1.38)$  &$6.01  (\pm 1.13)$     &$5.95 (\pm 1.50)$        &${\bf 4.84  (\pm 1.50)}$\\
             &Naive Bayes        &$11.27 (\pm 4.09)$ &$19.65 (\pm 15.14)$&$10.36 (\pm 5.14)$     &$18.00 (\pm 14.24)$      &${\bf 10.03 (\pm 3.53)}$\\
             &Random Forest      &$3.32 (\pm 1.23)$  &$3.19 (\pm 1.10)$  &$3.19 (\pm 1.13)$      &$3.18 (\pm 1.04)$        &${\bf 3.13 (\pm 1.04)}$\\ \midrule
heart        &Logistic Regression&$39.68 (\pm 7.90)$ &$40.55 (\pm 9.10)$ &$39.96 (\pm 7.40)$     &$39.94 (\pm 6.93)$       &${\bf 36.56 (\pm 8.29)}$\\
             &SVM                &$45.17 (\pm 6.98)$ &$39.94 (\pm 8.55)$ &$39.76 (\pm 8.49)$     &$39.74 (\pm 8.92)$       &${\bf 35.37 (\pm 6.84)}$\\
             &AdaBoost           &$30.87 (\pm 12.04)$&$29.24 (\pm 6.47)$ &$29.37 (\pm 12.19)$    &$31.27 (\pm 8.37)$       &${\bf 26.96(\pm 13.12)}$\\
             &Naive Bayes        &$22.79 (\pm 6.17)$ &$24.78 (\pm 7.99)$ &$22.87 (\pm 6.02)$     &$24.58 (\pm 7.98)$       &${\bf 21.97 (\pm 6.41)}$\\
             &Random Forest      &$20.87 (\pm 6.64)$ &$20.98 (\pm 6.62$  &$20.96 (\pm 6.67)$     &$21.96 (\pm 6.70)$       &${\bf 19.95 (\pm 6.61)}$\\ \midrule
diabetes     &Logistic Regression&$37.62 (\pm 4.35)$ &$40.22 (\pm 4.10)$ &$38.38 (\pm 3.85)$     &$40.11 (\pm 3.74)$       &${\bf 36.86 (\pm 4.81)}$\\
             &SVM                &$33.19 (\pm 5.69)$ &$37.22 (\pm 6.63)$ &$33.11 (\pm 6.45)$     &$33.38 (\pm 5.74)$       &${\bf 32.83 (\pm 5.62)}$\\
             &AdaBoost           &$37.69 (\pm 4.28)$ &$40.13 (\pm 5.28)$ &$40.76 (\pm 4.31)$     &$41.26 (\pm 5.16)$       &${\bf 33.45 (\pm 4.35)}$\\
             &Naive Bayes        &$39.29 (\pm 3.98)$ &$39.21 (\pm 3.18)$ &$39.26 (\pm 2.97)$     &$39.35 (\pm 2.85)$       &${\bf 38.10 (\pm 4.02)}$\\
             &Random Forest      &$30.09 (\pm 3.03)$ &$30.90 (\pm 3.52)$ &$31.07 (\pm 3.10)$     &$30.51 (\pm 3.67)$       &${\bf 29.46 (\pm 2.99)}$\\ \midrule
madelon      &Logistic Regression&$47.31 (\pm 1.80)$ &$47.80 (\pm 1.57)$ &$47.16 (\pm 1.68)$     &$46.81 (\pm 1.56)$       &${\bf 46.31 (\pm 1.69)}$\\
             &SVM                &$47.78 (\pm 1.53)$ &$47.28 (\pm 2.20)$ &$47.10 (\pm 2.13)$     &$47.12 (\pm 1.65)$       &${\bf 46.56 (\pm 2.12)}$\\
             &AdaBoost           &$42.92 (\pm 1.40)$ &$42.91 (\pm 1.68)$ &$43.36 (\pm 1.81)$     &$42.90 (\pm 1.40)$       &${\bf 40.64 (\pm 7.32)}$\\
             &Naive Bayes        &$41.90 (\pm 1.05)$ &$41.43 (\pm 8.76)$ &$41.79 (\pm 8.05)$     &$41.62 (\pm 7.43)$       &${\bf 41.03 (\pm 8.32)}$\\
             &Random Forest      &$35.90 (\pm 0.83)$ &$35.42 (\pm 1.75)$ &$35.13 (\pm 1.30)$     &${\bf 34.79 (\pm 1.75)}$ &$35.56 (\pm 2.03)$      \\
\bottomrule
\end{tabular}%
}
\end{table*}

\begin{figure*}[t]
    \centering
    \includegraphics[scale=0.42]{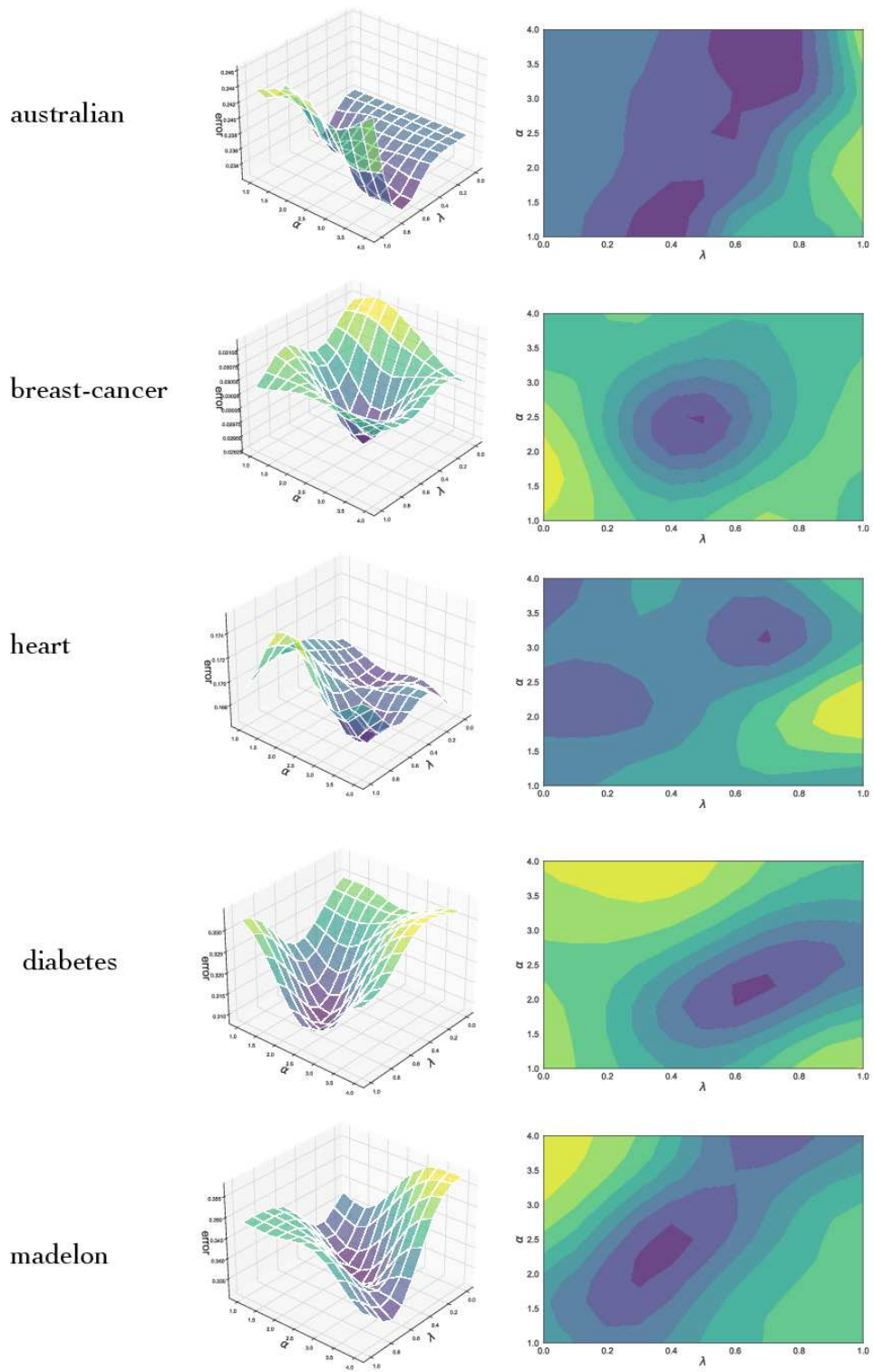}
    \caption{Visualization of grid search for $\alpha$ and $\lambda$ on LIBSVM dataset.
    For the sake of clarity, we apply a moving average.}
    \label{fig:gird_search_surface}
\end{figure*}


\begin{table*}[t]
\centering
\caption{Mean squared errors averaged over 10 trails on scikit-learn~\cite{scikit-learn} regression benchmark datasets. The numbers in the brackets are the
standard deviations. For the methods with (optimal), the optimal parameters for the test data are obtained by the linear search.
The lowest mean squared errors are shown with bold.}
\label{tab:sklearn_regression}
\resizebox{\textwidth}{!}{%
\begin{tabular}{c|c|c|ccccc}
\toprule
Dataset            & \#features & \#data  & unweighted          & IWERM               & AIWERM (optimal)   & RIWERM (optimal)    & ours                    \\ \midrule
boston             & $13$       & $506$   & $83.22 (\pm 5.72)$  & $69.87 (\pm 2.31)$ & $69.68 (\pm 1.46)$ & $69.96 (\pm 1.84)$ & ${\bf 68.36 (\pm 1.20)}$\\
diabetes           & $10$       & $442$   &$0.049 (\pm 0.007)$  &$0.0501 (\pm 0.009)$ &$0.049(\pm 0.008)$  &$0.049 (\pm 0.009)$  &${\bf 0.048(\pm 0.007)}$ \\
california housing & $8$        & $20,640$&$1.432 (\pm 0.095)$  &$1.3214 (\pm 0.345)$ &$1.260(\pm 0.125)$  &$1.261 (\pm 0.086)$  &${\bf 1.232(\pm 0.095)}$ \\

\bottomrule
\end{tabular}%
}
\end{table*}

\begin{figure*}[t]
    \centering
    \includegraphics[scale=0.55, bb=0 0 765 527]{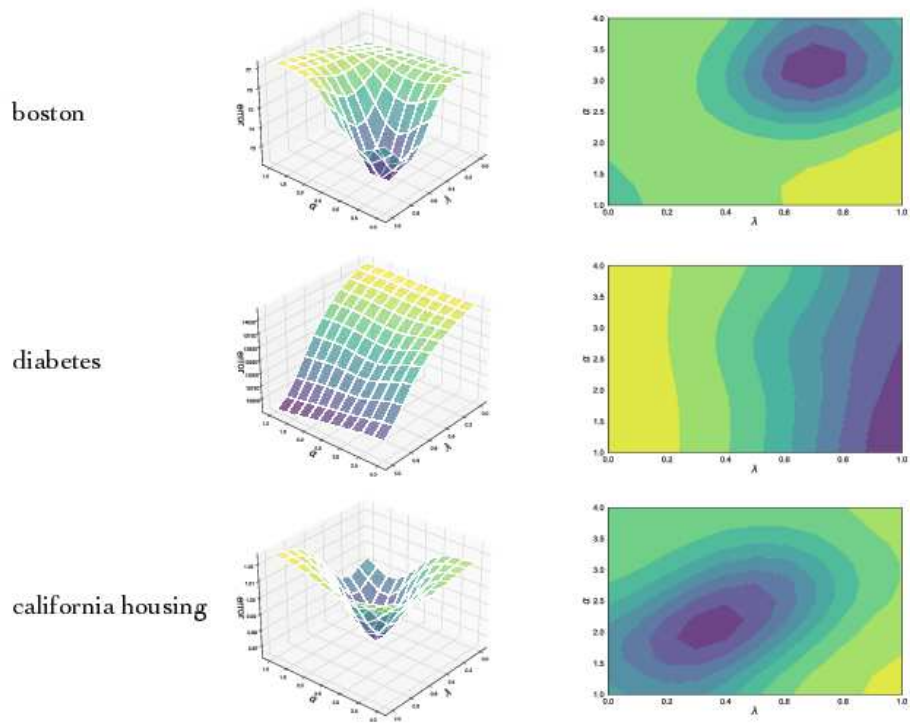}
    \caption{Visualization of grid search for $\alpha$ and $\lambda$ on scikit-learn regression dataset.
    For the sake of visualization, we apply a moving average.}
    \label{fig:gird_search_surface_regression}
\end{figure*}


\begin{table*}[t]
\centering
\caption{Mean misclassification rates averaged over 10 trails on scikit-learn~\cite{scikit-learn} multi-class classification benchmark datasets. The numbers in the brackets are the
standard deviations. For the methods with (optimal), the optimal parameters for the test data are obtained by the linear search.
The lowest misclassification rates among five methods are shown with bold.}
\label{tab:sklearn_multi_class}
\resizebox{\textwidth}{!}{%
\begin{tabular}{c|c|ccccc}
\toprule
Dataset & model               & unweighted         & IWERM               & AIWERM (optimal)    & RIWERM (optimal)    & ours                   \\ \midrule
digits  & Logistic Regression & $6.92 (\pm 2.25)$  & $7.10 (\pm 1.76)$   & $6.90 (\pm 1.81)$   & $6.89 (\pm 1.80)$   & ${\bf 6.79 (\pm 1.82)}$\\
        & SVM                 & $4.21 (\pm 2.04)$  & $3.94 (\pm 1.14)$   & $4.20 (\pm 1.22)$   & $4.02 (\pm 1.18)$   & ${\bf 3.88 (\pm 1.11)}$\\
        & AdaBoost            & $67.98(\pm 12.35)$ & $71.77(\pm 6.25)$   & $71.26 (\pm 8.21)$  & $70.47 (\pm 6.46)$  & ${\bf 65.20 (\pm 8.20)}$\\
        & Naive Bayes         & $19.07(\pm 2.45)$  & $18.68 (\pm 2.61)$  & $19.31 (\pm 2.64)$  & $18.78 (\pm 2.68)$  & ${\bf 18.58 (\pm 2.66)}$\\
        & Random Forest       & $6.85 (\pm 2.40)$  & $6.30 (\pm 1.80)$   & $6.23 (\pm 1.53)$   & $6.27 (\pm 1.64)$   & ${\bf 6.17 (\pm 1.90)} $\\ \midrule
iris    & Logistic Regression & $54.69(\pm 20.51)$ & $36.49 (\pm 22.39)$ & $45.78 (\pm 18.09)$ & $35.37 (\pm 23.19)$ & ${\bf 28.89 (\pm 20.54)}$\\
        & SVM                 & $55.16(\pm 22.60)$ & $36.65 (\pm 22.56)$ & $33.21 (\pm 20.25)$ & $30.46 (\pm 21.14)$ & ${\bf 29.04 (\pm 20.02)}$\\
        & AdaBoost            & $27.19(\pm 22.56)$ & $26.00 (\pm 22.98)$ & $26.00 (\pm 22.98)$ & $26.00 (\pm 22.98)$ & ${\bf 19.62 (\pm 21.36)}$\\
        & Naive Bayes         & $35.88 (\pm 23.24)$& $35.97 (\pm 26.79)$ & $37.99 (\pm 23.55)$ & $33.84 (\pm 25.64)$ & ${\bf 27.52 (\pm 21.00)}$\\
        & Random Forest       & $26.16 (\pm 22.86)$& $32.17 (\pm 21.21)$ & $26.00 (\pm 22.98)$ & $28.47 (\pm 22.63)$ & ${\bf 22.77 (\pm 21.74)}$\\ \midrule
covtype & Logistic Regression & $45.36 (\pm 13.08)$& $32.04 (\pm 5.833)$ & $30.62 (\pm 4.64)$  & $25.20 (\pm 8.99)$  & ${\bf 19.99 (\pm 5.56)} $\\
        & SVM                 & $-        $        & $-        $         & $-        $         & $-        $         & $-                      $\\
        & AdaBoost            & $47.53 (\pm 14.51)$& $25.55 (\pm 12.27)$ & $25.47 (\pm 14.14)$ & $27.86 (\pm 10.37)$ & ${\bf 18.96 (\pm 7.29)} $\\
        & Naive Bayes         & $41.13 (\pm 15.65)$& $30.66 (\pm 15.09)$ & $28.48 (\pm 15.66)$ & $27.20 (\pm 15.79)$ & ${\bf 19.64 (\pm 15.62)}$\\
        & Random Forest       & $23.51(\pm 3.31)$  & $18.18 (\pm 2.01)$  & $17.28 (\pm 2.05)$  & $17.13 (\pm 2.26)$  & ${\bf 16.42 (\pm 2.08)}$\\
\bottomrule
\end{tabular}%
}
\end{table*}

\begin{figure*}[t]
    \centering
    \includegraphics[scale=0.55, bb=0 0 755 518]{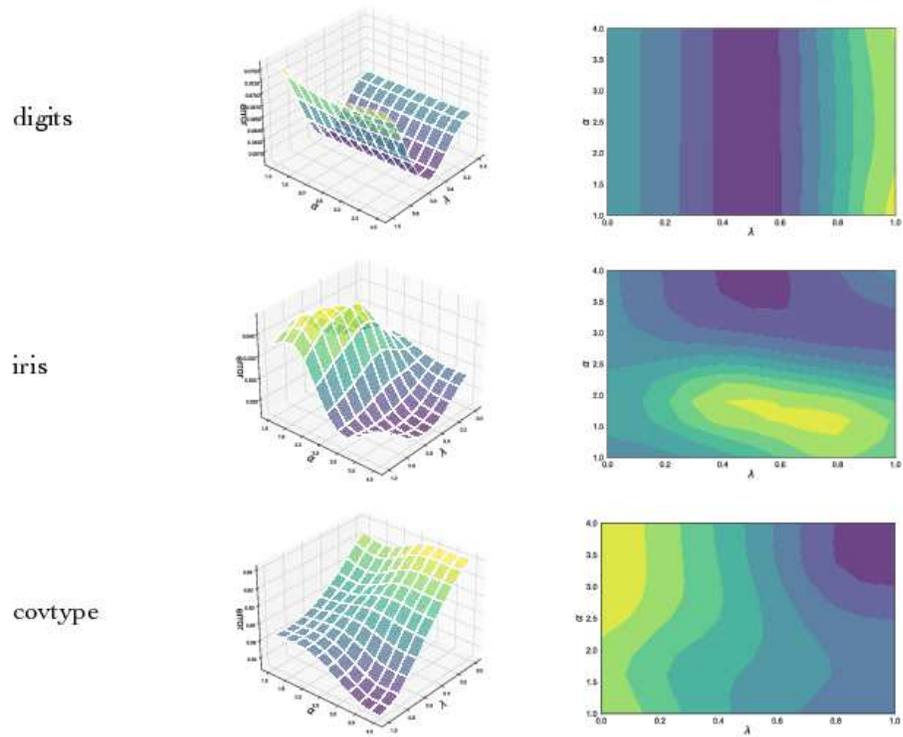}
    \caption{Visualization of grid search for $\alpha$ and $\lambda$ on scikit-learn multi-class classification dataset.
    For the sake of visualization, we apply a moving average.}
    \label{fig:gird_search_surface_multi_class}
\end{figure*}

\begin{figure*}[t]
    \centering
    \includegraphics[scale=0.45, bb=0 310 960 1415]{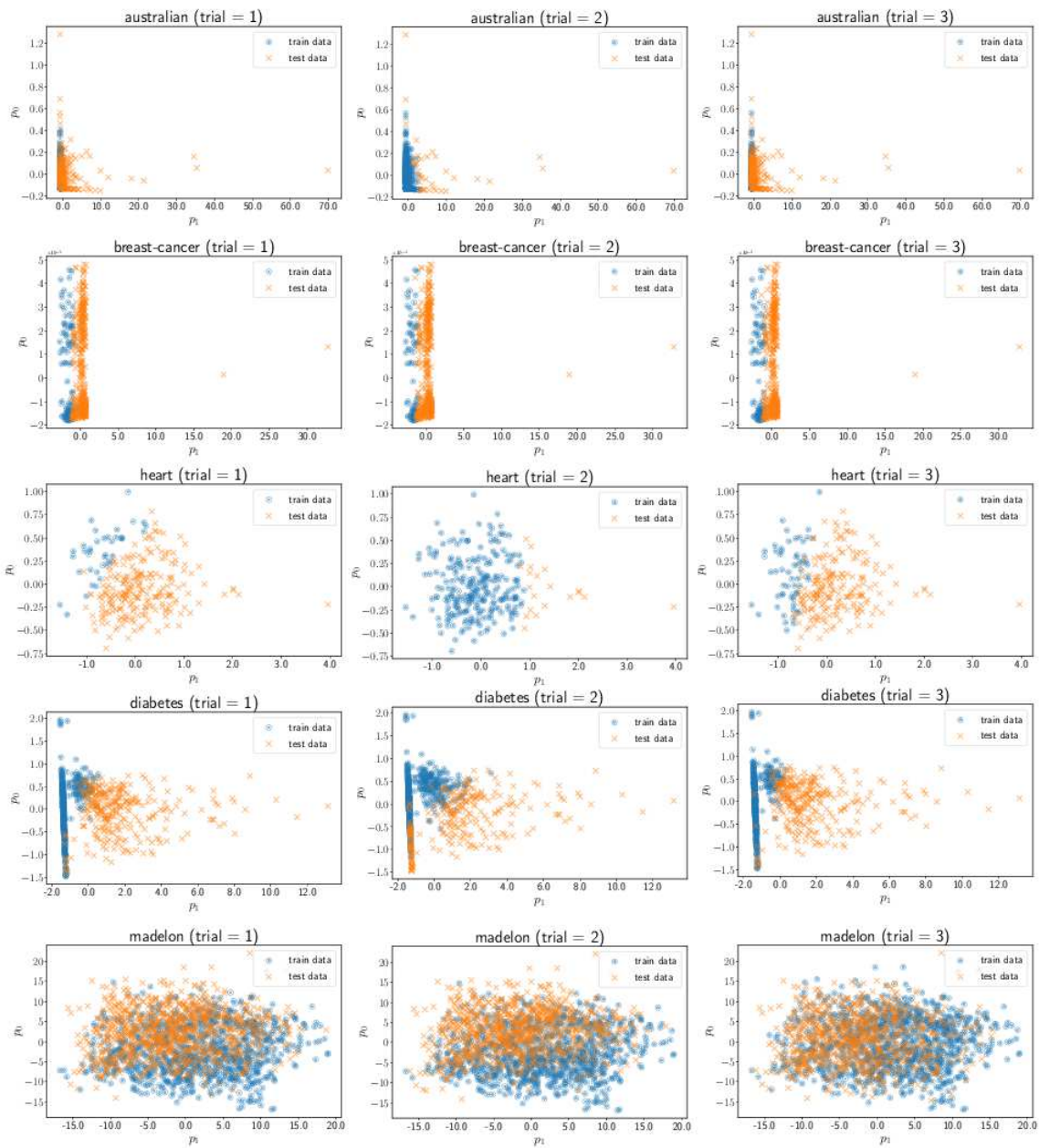}
    \caption{Plot of covariate shifts using the method of Cortes et al~\cite{Cortes2008-zd}. Each dataset is included in LIBSVM and mapped to two dimensions by PCA.}
    \label{fig:covariate_shift_libsvm}
\end{figure*}

\end{document}